\documentclass{article}

\PassOptionsToPackage{numbers}{natbib}

\usepackage[final]{neurips_2024}

\usepackage[utf8]{inputenc} % allow utf-8 input
\usepackage[T1]{fontenc}    % use 8-bit T1 fonts
\usepackage{graphicx,subcaption,lipsum}       % hyperlinks
\usepackage{url}            % simple URL typesetting
\usepackage{booktabs}       % professional-quality tables
\usepackage{makecell}
\usepackage{amsfonts}       % blackboard math symbols
\usepackage{nicefrac}       % compact symbols for 1/2, etc.
\usepackage{microtype}      % microtypography
\usepackage[dvipsnames]{xcolor}         % colors
\usepackage{amssymb,amsfonts,amsmath,amsthm,amscd,dsfont,mathrsfs,bbold,pifont}
\usepackage{graphicx}
\usepackage{comment}
\usepackage{caption, subcaption, float}
\usepackage{multirow}
\usepackage{enumitem}
\usepackage{array}
\usepackage{wrapfig}
\newcolumntype{P}[1]{>{\centering\arraybackslash}p{#1}}
\usepackage{hyperref}
\definecolor{mydarkblue}{rgb}{0,0.08,0.60}
  \hypersetup{ %
    colorlinks=true,
    linkcolor=mydarkblue,
    citecolor=mydarkblue,
    filecolor=mydarkblue,
    urlcolor=mydarkblue,
    }

\newtheorem{theorem}{Theorem}

\newtheorem{lemma}{Lemma}
\newtheorem{corollary}{Corollary}
\newtheorem{definition}{Definition}
\newtheorem{remark}{Remark}

\newtheorem{conjecture}{Conjecture}

\title{How Far Can Transformers Reason?\\
The Globality Barrier and Inductive Scratchpad}

\author{%
Emmanuel Abbe$^{1,2}$, Samy Bengio$^1$, Aryo Lotfi$^2$, Colin Sandon$^2$, Omid Saremi$^1$\\
$^1$Apple \quad $^2$EPFL 
}
\begin{document}

\maketitle

\begin{abstract}
Can Transformers predict new syllogisms by composing established ones? 
More generally, what type of targets can be learned by such models from scratch? Recent works show that Transformers can be Turing-complete in terms of expressivity, but this does not address the learnability objective. This paper puts forward the notion of {\it globality degree} of a target distribution to capture when weak learning is efficiently achievable by regular Transformers. %where the latter measures the least number of tokens required in addition to the tokens histogram to correlate nontrivially with the target.
This measure shows a contrast with the expressivity results of Transformers captured by $TC^0/TC^1$ classes (further studied here), since the globality relates to correlations with the more limited $NC^0$ class. 
We show here experimentally and theoretically under additional assumptions that distributions with high globality cannot be learned efficiently. In particular, syllogisms cannot be composed on long chains.  
Further, we develop scratchpad techniques and show that: (i) agnostic scratchpads cannot break the globality barrier, (ii) educated scratchpads can break the globality with intermediate steps, although not all such scratchpads can generalize out-of-distribution (OOD), (iii) a notion of `inductive scratchpad', that composes the prior information more efficiently, can both break the globality barrier and improve the OOD generalization. In particular, some of our inductive scratchpads can achieve length generalizations of up to $6\times$ for some arithmetic tasks depending on the input formatting.  
\end{abstract}

\vspace{-2mm}
\section{Introduction}
Transformers \cite{vaswani2017attention-transformer} have proved to have strong learning capabilities, in particular in applications with large amounts of text, image, or audio data \cite{dosovitskiy2020image, alabdulmohsin2022revisiting}. 
Some reasoning capabilities are also notable in these settings, however, the picture deteriorates when the target complexity increases, such as in tasks involving more advanced forms of `reasoning' 
\cite{saxton2019analysing, lewkowycz2022solving-minerva,Zhang2021PointerVR,johnson2017clevr,bakhtin2019phyre,velivckovic2022clrs,mahdavi2022better}. While reasoning is present at all levels of learning, it is pushed to a higher level in tasks such as logic or mathematics, where `learning by seeing enough representative examples' is precluded by the more combinatorial nature of the task. 
For such tasks, combining learned concepts in order to extrapolate seems necessary, as for the length generalization problem \cite{anil2022exploring-length}. 
Current Transformer-based models exhibit difficulties learning at scale on such tasks.
Can we understand why and what is missing? We start with a specific motivational example before expanding the discussion to more general tasks. 

\subsection{Syllogisms composition}
Reasoning relates to the process of inferring new knowledge by composing efficiently some prior knowledge. A basic notion of reasoning is syllogism composition, e.g., inferring $a \Rightarrow c$ from $a \Rightarrow b$ and $b \Rightarrow c$. For instance, one may be given a set of implications:
{\footnotesize	
\vspace{-.4cm}
\begin{center}
\begin{tabular}{P{0.47\textwidth}|P{0.47\textwidth}}
 task 1 has priority over task 2 &   $x>2 \quad \Rightarrow \quad x^2 > 3$ \\
 task 1 has priority over task 3   &  $x > 2 \quad \Rightarrow \quad (x-1)(x+1) > 1$ \\
task 4 has priority over task 1  &  $4^x >17 \quad \Rightarrow \quad x > 2$  \\
 task 1 has priority over task 5   &  $x > 2 \quad \Rightarrow \quad  x-x^4 > 1-2x^4$ 
\end{tabular}
\end{center}
}
\vspace{-.1cm}
and without additional background information, one would like to know using logic whether
\vspace{-.2cm}
{\footnotesize
\begin{center}
\vspace{-.3cm}
\begin{tabular}{P{0.47\textwidth}|P{0.47\textwidth}}
task 3 has priority over task 5?   &  $4^x >17 \quad \stackrel{?}{\Rightarrow} \quad x^2> 3$.
\end{tabular}
\end{center}
}
The goal here is to identify whether a syllogism can be composed\footnote{Answering `yes/1' if the syllogism can be obtained by composing input ones or `cannot tell/0' otherwise. }  by prior ones. Simplifying the input format, the above correspond to identifying paths in token sequences describing the directed edges of an underlying graph, i.e.,  
whether there is a directed path $3 \to 5$ (case 1) or $4 \to 2$ (case 2) using the directed edges $\{(1 \to 2), (1 \to 3), (4 \to 1), (1 \to 5)\}$. 

\begin{comment}
Another form of such task may concern equivalances, e.g., in a maths context:
\begin{center}{\it 
$x>2 \quad \Leftrightarrow \quad x^3 > 8$\\
$x^3 > 8 \quad \Leftrightarrow \quad 2x^3 > 16$ \\
$2x^3 > 16 \quad \Leftrightarrow \quad x^3 > 16 - x^3$ \\
$2x^3 > 16 \quad \Leftrightarrow \quad  2x^3-8 > 8$ 
}
\end{center}
and one would like to know whether these equivalances imply the following one without additional maths background: 
\begin{center}{\it 
$ x^3 > 16 - x^3 \quad \Leftrightarrow \quad x > 2$ 
\,\,\,?}
\end{center}
This task corresponds to identifying  
whether there is a path between $1$ and $3$ given the undirected edgeset $\{(1,2), (2,3), (3,4), (3,5)\}$.
Here the answer is yes. 
\end{comment}
This type of task is nontrivial for current LLMs, and we refer to Appendix \ref{app:gpt} for experiments with GPT models.\footnote{At the time of the experiments, ChatGPT was in particular not successful at these two tasks.} Note that here we are not interested specifically in solving a graph-based task, but rather in understanding when Transformers can compose and more generally how far they can do so. 
We would like to identify particular measures on the data distribution (e.g., syllogisms topologies in the above example) that capture when Transformers can efficiently learn.

\begin{comment}
\begin{figure}
    \centering
\includegraphics[width=.9\textwidth]{graphs_syllogism.pdf}
    \caption{Syllogisms: implications and equivalances}
    \label{examples}
    \vspace{-.3cm}
\end{figure}
\end{comment}

\subsection{Hardness of long compositions}\label{hardness}
Consider the previous syllogism composition task where implications are drawn on a graph with 24 edges drawn randomly over 24 vertices. Picking vertices at distances 1 to 4 for the connected case and picking disconnected vertices uniformly at random lets a Transformer achieve a test accuracy of more than 80\% after about 2000 iterations.  
However, does this mean that the model has learned to compose syllogisms, or has it found shortcuts, e.g., based on node degrees, to guess the implications often enough? In Appendix~\ref{sec:ER}, we provide empirical evidence supporting the latter. Motivated by this issue and to preclude spurious correlations, we consider the following distribution. 

\begin{definition}[Cycle task]\label{def:cycle-task}
For $n \ge 1$, consider the binary classification task with equiprobable classes defined by
\vspace{-2mm}
\begin{enumerate}
\item Class 1: a graph uniformly drawn on $2n$ vertices with two disjoint cycles of length $n$ and a pair of vertices in disjoint cycles queried for path;
\vspace{-2mm}
\item Class 2: a graph  uniformly drawn on $2n$ vertices with one cycle of length $2n$ and a pair of vertices at distance $n$  queried for path.
\end{enumerate}
\vspace{-2mm}
The input of this task is the graph edge set with the queried vertices.  
The label is 0 if the two queried vertices are not connected (Class 1) and 1 if they are (Class 2). 
See Figure~\ref{fig:2cycles} for an illustration.

%\begin{enumerate}
%\item the graph edgeset, e.g., $\{(1\to 3),(2\to 4),(3\to 2),(4\to 1),(1\stackrel{?}{\to}2)\}$ or $(a>c;b>c;b>d;d>a;a?c)$ in the example of Figure \ref{};
%\item the graph adjacency matrix, e.g.,  $(0,0,1,1;0,0,1,1;1,1,0,0;1,1,0,0)$ in the example of Figure \ref{}.
%\item Rep.\ 3: the graph edgeset ordered in a way that edges follow each other, e.g.,  $(1,3;3,2;2,4;4,1)$.
%\end{enumerate}
\end{definition}

Figure \ref{fig:no-cot-fails} shows that the learning complexity increases `exponentially' as $n$ grows using  GPT2-style Transformers of more than $10M, 25M, 85M$ parameters; e.g., the $10M$ model fails to learn for $n \geq 7$ in $100k$ iterations. Why is that? 
Can a larger scale further help here?
     \begin{center}
{\it Can a large (poly-size) Transformer learn the cycle task when $n$ gets large? If not, why so?}
\end{center}

\begin{figure}[t]
\centering
\begin{subfigure}[b]{.5\textwidth}
  \centering
  
\includegraphics[width=1.0\textwidth]{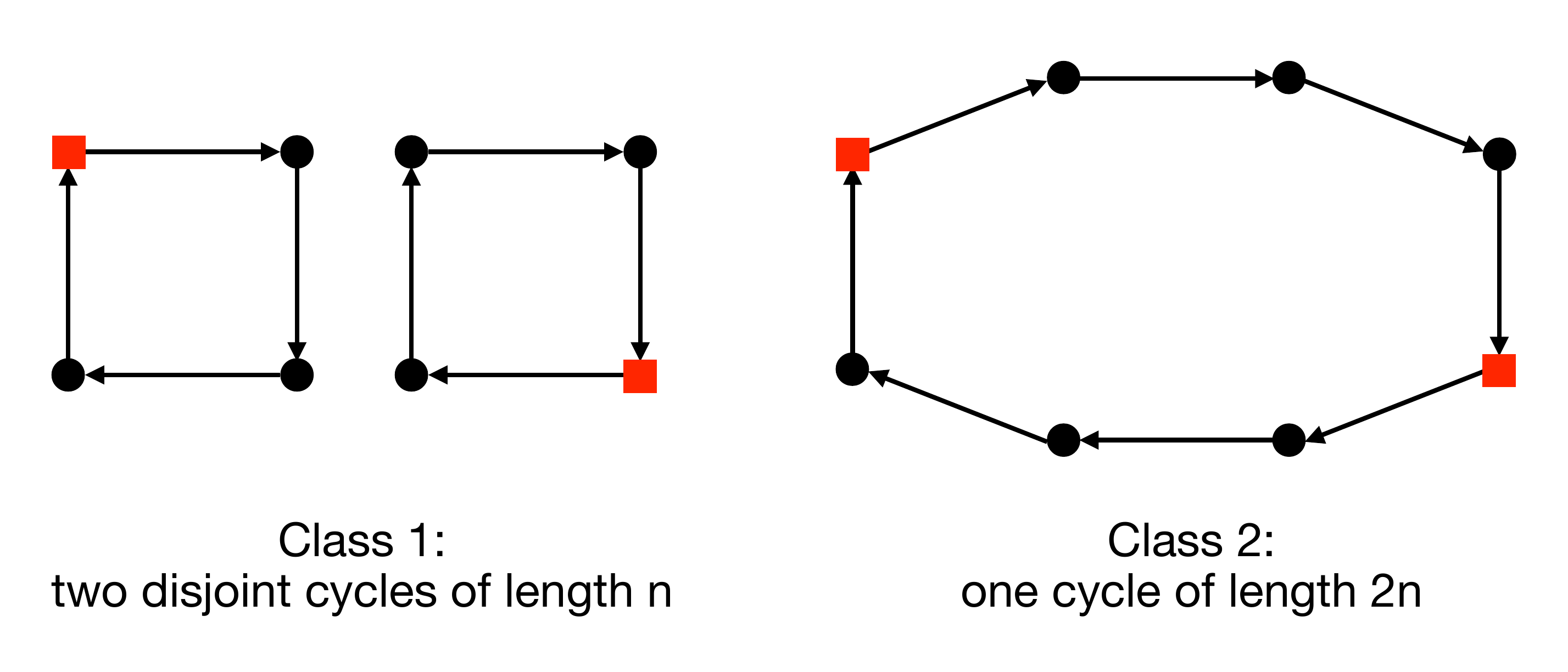}\caption{Cycle task: binary classification to predict whether two vertices (the red squares) are connected or not on the above two graph topologies.}
    \label{fig:2cycles}
    %\vspace{-.3cm}
  
\end{subfigure}%
\hfill
\begin{subfigure}[b]{.47\textwidth}
  \centering

\includegraphics[width=0.9\textwidth]{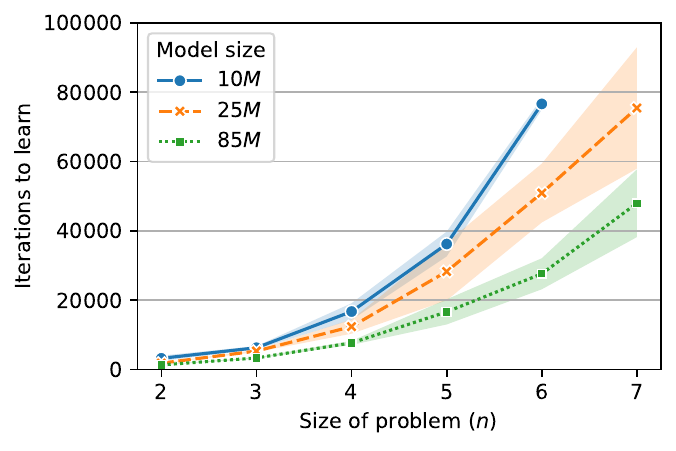}\caption{Nb.\ of iterations with $512$ fresh samples to learn ($\ge 95\%$ acc.) the cycle task of size $n$ using  GPT2-style models with $10M, 25M, 85M$ parameters. 
    %We were not able to learn the task for $n \geq 14$ and higher up to $120000$ iterations.
    }
    \label{fig:no-cot-fails}
  
\end{subfigure}
\caption{ Illustration of the cycle task for $n=4$ (left) and the complexity to learn it (right).}
\label{fig:cycle-task-illust-complexity}
\end{figure}

%For representation 3, the model can simply check whether the first symbol repeats before the end of the string, or even more precisely, if the $n$-th token matches the second token. I.e., $(1,3;3,1;2,4;4,2)$ gives two cycles since 1 repeats, but $(1,3;3,2;2,4;4,1)$ gives one cycle. Of course, using this representation is a cheat, as ordering the edges to form a path is essentially what the task is about. 

A challenge for the cycle task is that there is no clear `low-complexity pattern' in the input representation that indicates whether there are 1 or 2 cycles. No simple statistics based on degrees, edge counts, or finite motif counts that can tell if the vertices are connected or not. One has to consider at least $n$ edges in order to get any correlation with the presence of a path. In other words, the task requires a `global reasoning' involving a `large' number of input tokens and this seems hard for Transformers.

%Our globality theory will allow to explain such outcomes. 

\subsection{Hardness of global reasoning}
As discussed, the cycle task appears to be challenging for Transformers as it requires some global reasoning. Other tasks such as subset parities exhibit the same challenge. However the latter can be proved to be not efficiently learnable by various regular neural networks and noisy gradient descent, as one can get explicitly a class of functions (through orbit arguments \cite{abbe2020poly,non-univ}) that has large statistical dimension \cite{kearns} or low cross-predictability \cite{abbe2020poly,AS20} (see Appendix \ref{lower_bound_ref}). For the cycle task, we have a single distribution, and it is unclear how to use the invariances of Transformers to get arguments as in \cite{abbe2020poly,non-univ}, as the input distribution is not invariant under the symmetries of the model. We thus would like to develop a more general complexity measure that unifies why such tasks are hard for Transformer-like models and that formalizes the notion of `global reasoning barrier' when models are trained from scratch. We also would like to understand how the scratchpad methodologies that have proved helpful in various settings (see Section \ref{sec:scratchpad}) can help here. This raises the questions: 
\begin{center}
{\it (1) How can we formalize the `global reasoning barrier' in general terms?}\\
{\it (2) Can we break the `global reasoning barrier' with scratchpad methodologies?}
\end{center}

\subsection{Our contributions}
We provide the following contributions:
\begin{itemize}[leftmargin=*]
  \item We introduce the notion of {\it globality degree} in Definition \ref{def:globality} to capture when weak learning is efficiently achievable by Transformers. 
    The globality degree measures the least number of tokens required in addition to the token histogram to correlate nontrivially with the target; it is also related in Lemma \ref{NC0_lemma} to correlations with $NC^0$ circuits, showing the contrast between learnability and expressivity controlled by $TC^0/TC^1$ with constant/logarithmic depth \cite{parallelism}. It is an explicit measure that applies to a data distribution without requiring an orbit argument to infer a class of distributions  \cite{abbe2020poly,non-univ}, giving a tight proxy for models like Transformers. 
We provide the following results based on the globality degree: 
    \begin{itemize}[leftmargin=*]
    \item A general conjecture (Conjecture \ref{gen_conj}), backed by experimental results, that claims efficient weak learning is achievable by a regular Transformer if and only if the globality degree is constant.  
    \item Theorem \ref{3cycleTheorem} that proves the negative side of the above conjecture, the {\it globality barrier}, in an instance of the cycle task under certain technical assumptions. (The cycle task is also put forward in the paper as a simple benchmark to test the global reasoning capabilities of models.) 
    \end{itemize}

    \item We then switch to the use of `scratchpads' to help with the globality barrier:
  \begin{itemize}[leftmargin=*]
    \item Agnostic scratchpad: we extend Theorem \ref{3cycleTheorem} to cases where a poly-size scratchpad is used by the Transformers, without any direct supervision of the scratchpad (i.e., the scratchpad mainly provides additional memory/compute). This shows that efficient weak learning is still not possible with such an agnostic scratchpad if the globality is non-constant. An educated guess about what to learn in the scratchpad based on some target knowledge is thus required.  
    \item Educated scratchpad: we generalize the measure of globality to the  `autoregressive globality' to quantify when an educated scratchpad is able to break the globality of a task with subtasks of lower globality. We give experimental results showing that educated scratchpads with constant autoregressive globality allow Transformers to efficiently learn tasks that may originally have high globality. This gives a way to measure how useful a scratchpad can be to break a target into easier sub-targets.     
    \item We introduce the notion of {\it inductive scratchpad}, a type of educated scratchpad that exploits `induction' compared to a fully educated scratchpad and thus composes more efficiently the prior state information. We show that when the target admits an inductive decomposition, such as for the cycle, arithmetic, or parity tasks, the inductive scratchpad both breaks the globality and improves the OOD generalization in contrast to fully educated scratchpads. This gives significant length generalization on additions (from 10 to 18 or from 4 to 26 depending on the method) and parities (from 30 to 50-55). For instance, using different methods, \cite{bubeck} can length generalize from 10 to 12 digits for additions, and  \cite{anil2022exploring-length} can get roughly 10 extra bits for parities. 
\end{itemize}
\end{itemize}

\section{Results on the global reasoning barrier}
{\bf Prior literature.} Much work in the literature has been devoted to complexity measures for the sample/time complexity of learning. The largest portion is devoted to target classes in PAC settings, e.g., with the VC dimension measures \cite{shaishai}, and some to statistical query (SQ) settings with the statistical dimension measures \cite{kearns,feldman_gen}. Here, we are however interested in measures that are relevant to (1) regular Transformers trained by (S)GD, and (2) data distribution fixed by a task. Some recent works have studied complexity measures for (S)GD-trained neural networks. Various settings and measures have been used, such as the noise sensitivity \cite{o'donnell_2014,Zhang2021PointerVR, hahn2024sensitive}, the cross-predictability \cite{abbe2020poly,AS20},  the NTK alignment~\cite{jacot2018neural,cortes}, the INAL \cite{AbbeINAL}, the $G$-alignment \cite{non-univ}, the information and generative exponents \cite{arous2021online,bruna1,bruna2} and the leap \cite{abbe2023leap}; we refer to Appendix \ref{lower_bound_ref} for discussions on these.  

However, despite this significant body of work, finding a simple measure giving a tight proxy for Transformer weak learning (i.e., the first non-trivial learning requirement) on a given data distribution, remains unsettled. We next propose such a measure.

\subsection{Defining globality and auto-regressive globality}
We define now the notion of globality degree, which in turn will quantify the notion of globality (or global reasoning) barrier. 
\begin{definition}\label{def:globality}{\bf(Globality degree)} 
For (a sequence of) distributions $D$ on $\mathcal{A}^n \times \mathcal{A}$, where $\mathcal{A}$ is a finite alphabet set of poly$(n)$-cardinality, define the globality degree of $D$, $\mathrm{glob}(D)$, as the smallest number of variables $k \in [n]$ for which there exists $S$, $|S|=k$ such that $$I(X[S],\hat{P}_X;Y)=n^{-O(1)}$$ where $(X,Y) \sim D$ and $\hat{P}_X$ is the empirical measure of $X$ (i.e., the histogram of tokens in $X$). 
\end{definition}

\begin{remark}
The globality degree, or simply globality, of a distribution measures the least number of input tokens to attend to in order to correlate non-trivially with the label when also given the histogram of tokens. The specific choice of the mutual information is not crucial, but one must use a proper measure of dependency (i.e., not just linear correlations), and the mutual information can have convenient chain rule properties. The definition can be related to correlations with $NC^0$ circuits (besides for the histogram requirement, see Lemma~\ref{NC0_lemma}) and also to low-degree polynomial testing, except that it is more general than the latter as it applies to arbitrary token space (without requiring polynomial definitions). Finally, we require the globality to achieve an inverse-polynomial mutual information, the weakest form relevant to weak learning with an inverse-polynomial edge, but one may naturally define the stronger notion with a mutual information of $\Omega(1)$.  
\end{remark}

We now define the globality in the autoregressive setting. 
\begin{definition}\label{def:seq-globality}{\bf (Globality degree in autoregressive setting)}
For $D$ on $\mathcal{A}^n \times \mathcal{A}^m$,  define $\mathrm{glob}(D)$ as the smallest integer $k$ for which there exist sets of indices $S_1,\ldots,S_m$, $|S_t| \le k$ for all $t \in [m]$, such that $$I((X,Y_{<t})[S_t],\hat{P}_{X,Y_{<t}};Y_t)=n^{-O(1)},$$ where $(X,Y) \sim D$ and $\hat{P}_{X,Y_{<t}}$ is the empirical measure of $(X,Y_{<t})$.
\end{definition}
In the auto-regressive setting, the globality is mostly relevant when weak learning gives strong learning, in order to let the scratchpad learn each step. 

As we will see in the next section, the globality degree is put forward as a tight proxy to understand efficient weak learning of regular Transformers for arbitrary data distributions. We first present the operational advantages of the definition, going back to the running example of the cycle task.

\paragraph{Attributes of $\mathrm{glob}$ and some examples.}
The globality has the attributes of being (i) a fairly explicit measure, (ii) applicable to any data distribution on tokens without having to infer a distribution class from the model invariances to estimate the distribution complexity, (iii) not limited to i.i.d.\ inputs but any input distribution, (iv) relevant to current models of interest such as Transformers.   

In particular, back to the cycle task, we have that any set of $n-1$ edges have the same distribution in Class 1 or 2, therefore the globality is at least $n$: 

\begin{lemma}\label{glob-cycle}
We have $\mathrm{glob}(\text{Cycle task}(n))\ge n$. 
\end{lemma}
As discussed in the next section, this explains why the cycle task is hard to learn. 
In contrast, the example at the beginning of Section \ref{hardness} has a much lower globality, as being connected correlates to query nodes having large enough degrees, and thus it can be expected for the model to learn with non-trivial accuracy (e.g., by using degree shortcuts).

\subsection{Transformers require low globality: formal results}
\vspace{-.1cm}
\begin{definition}
A neural network of input dimension $n$ (i.e., a directed acyclic graph with $n$ inputs) and depth $d$ (i.e., the longest path length from input to output) is
\begin{enumerate}
%\item is poly-regular if it has edge weights polynomial in $n$, $poly(n)$ number of edges and a regular\footnote{I.e., an activation function such that there exists constant $c$ such that the activation function is differentiable with derivative of absolute value less than $c$ almost everywhere.} activation function. In particular, classic fully connected MLPs with Gaussian initialization and Transformers as used here and in \cite{vaswani2017attention-transformer,gpt2} are poly-regular.
\vspace{-.2cm}
\item T-regular if it is a Transformer (e.g., \cite{vaswani2017attention-transformer, gpt2}) of polynomial size with Normal Gaussian i.i.d.\ positional embeddings, Normal Gaussian i.i.d.\ weights, and bidirectional attention.
\vspace{-.2cm}
\item T-regular with $s$-scratchpad if we have a constant sized token alphabet $\Sigma$ and a T-regular Transformer that takes an $m$-sequence in $\Sigma$ and outputs a probability mass function on $\Sigma$ (with well-behaved\footnote{I.e. there exists $c$ such that for any two vectors $u$ and $u'$ the total variation distance between the softmax's probability distributions on these vectors is at most $c||u-u'||_2$. Note that this does still allow a softmax which returns a specific value with probability $1$ on some vectors.} softmax). To compute the value of this net and scratchpad on an input $X\in \Sigma^n$ we set $X_n=X$ and then for each $m \in \{n,\ldots,n+s-1\}$ draw $x_{m+1}$ from the probability distribution represented by the net's value on $X_{m}$ and set $X_{m+1}=X_m\circ x_{m+1}$ (concatenation). Then, we consider $x_{n+s}$ to be the overall output of the net.\footnote{Normally when using a scratchpad we would use causal masking to make it so that we do not need to recalculate previous entries of lists in the Transformer whenever a new entry is added to the scratchpad. We use this one for simplicity, but we still expect all conjectures involving a scratchpad to hold when we apply causal masking to the scratchpad, and Theorem \ref{agnostic3cycleTheorem} and its proof are unchanged.}
\end{enumerate}
\end{definition}
%\vspace{-.1cm}

%\begin{remark}
%The important part of this definition for our purposes is that the probability distribution of the function computed by such a Transformer is invariant under permutations of the inputs, and if it is trained in a reasonable way on samples drawn from a probability distribution drawn from a class that is symmetric under permutations of the inputs its probability distribution will retain its symmetry under permutations of the inputs. That means that it is exactly as hard for it to learn any given function as it is for it to learn a function obtained by rearranging that function's inputs, so we can infer results about its difficulty learning a specific function from results about its difficulty learning a random element of that function's orbit.
%\end{remark}
%We now first state a result about the expressivity of such Transformer models, mostly proved in \cite{}.
\begin{remark}
In this paper, we focus on learnability via descent algorithms, but for the simpler question of expressivity, one wants to know whether there is any choice of parameters for which a Transformer can compute a target function
%rather than asking about the behavior of a transformer with random parameters. 
(with limits to how precisely it can record values).\\
(i) In \cite{parallelism},  the expressivity of Transformers with constant alphabet size and values recorded to inverse-polynomial accuracy is investigated. It is shown that such a Transformer of constant depth was limited to computing functions in $TC^0$. Conversely, it also showed that for any $TC^0$ function, there is a constant-depth Transformer and instruction string such that when the Transformer is given the instruction string and $x$ as its input it computes the function on $x$. The same technique would extend to show that with logarithmic depth, one can reach $TC^1$ (which includes connectivity tasks). \\
(ii) In \cite{expressivity}, well-behaved\footnote{They assumed a constant alphabet size, inverse-polynomial accuracy, constant depth, causal masking, averaging hard attention, and a projected pre-norm. We suspect that all of these requirements other than the constant alphabet size and constant depth could be dropped without changing the results we state here.} Transformers with a scratchpad are considered. It is shown that a Transformer with a scratchpad of logarithmic length is limited to computing functions in logspace. We tighten this to $TC^0$ in Lemma \ref{logspace} in Appendix~\ref{app:circuit-complexity-connections}.  
On the other hand, \cite{expressivity} also shows that any function in $P$ is computable by a Transformer with a scratchpad of polynomial length.\\
(iii) If we allow Transformers of poly depth then we can convert any poly-sized circuit to a Transformer by replacing each gate in the circuit with an attention head that attends to the values of the appropriate input tokens or attention heads and performs the appropriate computation on them. That means any function in P (including the cycle task) is computable by a Transformer of poly depth and size. 
\end{remark}

We now state the general conjecture putting forward the globality barrier for learning. 
\makeatletter
\begin{conjecture}\label{gen_conj}
A distribution $P_{X,Y}$ with well-behaved\footnote{Conditions such as assuming that there is no value of $X$ that is frequent enough that the model weakly learns the function simply by memorizing the value of that input. Most distributions of interest are well-behaved (including the cycle task). See Definition \ref{def:well-behaved} in Appendix \ref{app:conj-specification} for a formal definition and additional discussion.}  
$P_X$ is efficiently weakly\footnotemark
\global\let\saved@Href@A\Hy@footnote@currentHref
learnable by a T-regular Transformer if and only if\footnotemark \global\let\saved@Href@B\Hy@footnote@currentHref $P_{X,Y}$ has constant globality. 
\addtocounter{footnote}{-1}
\end{conjecture}
\let\Hy@footnote@currentHref\saved@Href@A
\footnotetext{I.e., with an accuracy that improves on the trivial accuracy by at least $n^{-c}$ for a constant $c$.}
\addtocounter{footnote}{1}
\let\Hy@footnote@currentHref\saved@Href@B
\footnotetext{More specifically, for the converse statement (learning due to constant globality) assume that $|\mathcal{A}|$ is constant.}
\makeatother
\begin{remark}
(1) An essential property of the model for the above conjecture is that the probability distribution of the function computed by the model is invariant under permutations of the inputs, and if it is trained reasonably on samples drawn from a distribution drawn from a class that is symmetric under permutations of the inputs, its distribution will retain its symmetry under permutations of the inputs. %That means that it is exactly as hard for it to learn any given function as it is for it to learn a function obtained by rearranging that function's inputs, so we can infer results about its difficulty learning a specific function from results about its difficulty learning a random element of that function's orbit. 
For MLPs, we expect most of the results in this paper to apply, with the modification of the globality not having access to the empirical measure of $X$, since one has additional symmetry obtained by exchanging tokens.  
(2) If we were to use causal masking or a specific choice of positional embeddings that would make it easier for the Transformer to focus on specific relevant subsets of the inputs, 
% which may be engineered using some prior knowledge of the target, 
one could potentially learn functions with higher globality. For instance, we would expect to be unable to learn a function that computes the parity of some subset of $\log(n)$ input bits. However, if we had a positional embedding that gave one value for all of the active bits and mapped other bits to $0$, then the Transformer would probably be able to learn the function in question. Likewise, causal masking makes it so that the first few elements of any list the Transformer computes depend only on the first few tokens, which makes it easier to learn functions that would rely on those symbols. 
\end{remark}

We prove the negative side of Conjecture \ref{gen_conj} for a variant of the cycle task.
\begin{theorem}\label{3cycleTheorem}
Let $G$ be a directed graph which consists of a cycle of length $3n$ with probability $2/3$ and~3 cycles of length $n$ otherwise. Next, if there are $3$ cycles pick one vertex from each and if there is one cycle pick 3 vertices that are each $n$ edges apart. Then, label these vertices with $a\_0$, $b\_0$, $c\_0$ uniformly at random. Next, number every other vertex by the distance from one of these three to it, and for each $i$, label uniformly at random the vertices at distance $i$ by $a\_i$, $b\_i$, and $c\_i$. Finally, store the edges between $a\_{i-1},b\_{i-1},c\_{i-1}$ and $a\_i, b\_i,c\_i$ in $X$ (as described in Figure \ref{fig:3cycles}), and let $Y$ report whether $a\_0, b\_0,c\_0$ are in the same cycle or not. Then if we train a T-regular neural network on $(X,Y)$ generated in this manner using population\footnote{This would also be true for batch GD with batches of size $n^c$, $c$ dependent on the other hyperparameters.} gradient descent with polynomial hyperparameters\footnote{I.e., either polynomial learning rate, polynomial clipping \cite{abbe2020poly,abbe2021power}, and weights stored using a logarithmic number of bits of precision and random rounding: for $a<b<c$ if $b$ needs to be rounded to $a$ or $c$ then it rounds to $c$ with probability $(b-a)/(c-a)$, or with polynomial learning rate, polynomial clipping and polynomial noise added to the gradients.} and a differentiable loss function then the network fails to weakly learn.
\end{theorem}
The proof of Theorem \ref{3cycleTheorem} is presented in  Appendix \ref{app:proof-thm1}.

\begin{figure}
\centering
\includegraphics[width=.45\textwidth]{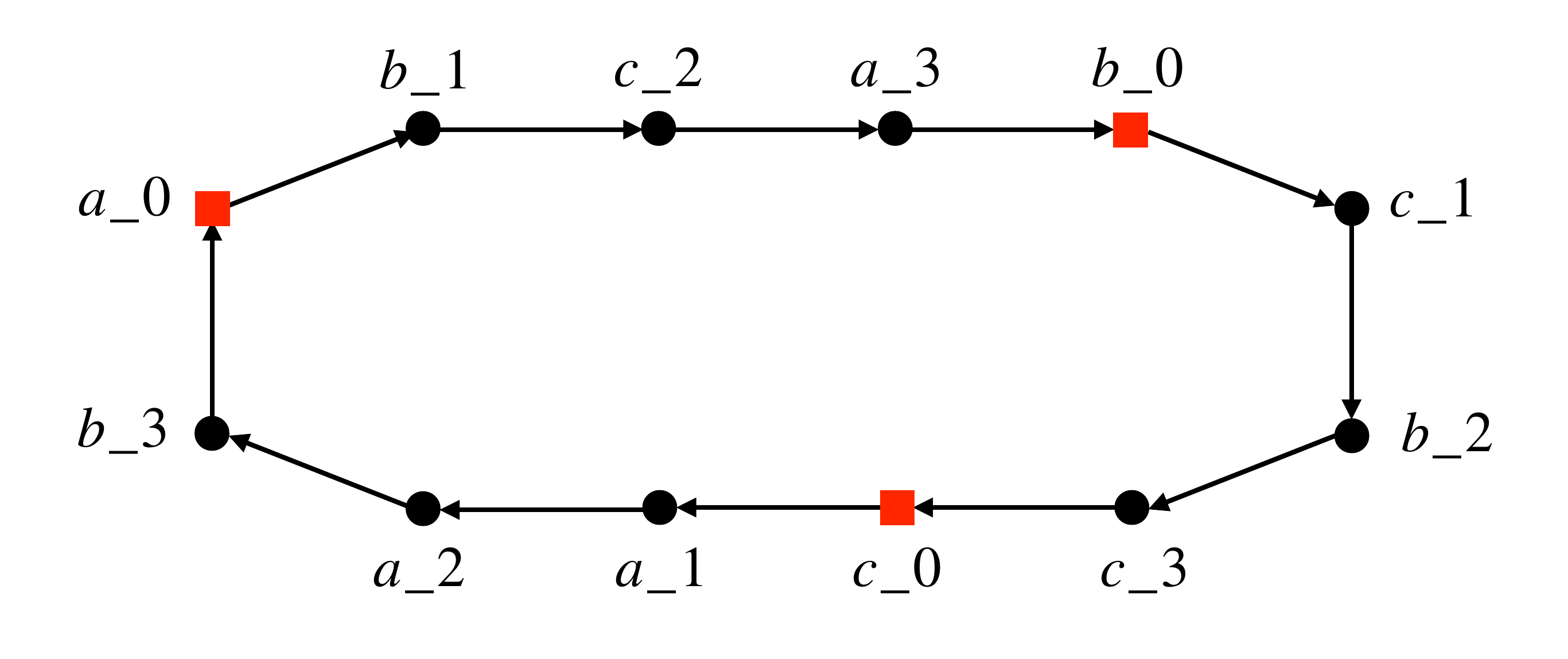}
    \caption{The cycle task variant used in Theorem \ref{3cycleTheorem}: the above example is stored as {\scriptsize \texttt{a\_0>b\_1;b\_0>c\_1;c\_0>a\_1;a\_1>a\_2;b\_1>c\_2;c\_1>b\_2;a\_2>b\_3;b\_2>c\_3;c\_2>a\_3;a\_3>b\_0;b\_3>a\_0;c\_3>c\_0;a\_0?b\_0?c\_0}}}
    \label{fig:3cycles}
\end{figure}

\begin{comment}
 \[X={\color{red} \bigcirc_{i=0}^{n-1}(} a\_{\color{red} i}>{\color{red} e(a\_i)}\_{\color{red} (i+1)};b\_{\color{red} i}>{\color{red} e(b\_i)}\_{\color{red} (i+1)};c\_{\color{red} i}>{\color{red} e(c\_i)}\_{\color{red} (i+1)} {\color{red} )} a\_0?b\_0?c\_0\]
where $e(v)$ represents the letter of the vertex that $v$'s edge points to, all of the instances of $i$ or $i+1$ should have the appropriate value substituted in and the symbols in black should be used exactly as stated
\end{comment}

\subsection{Agnostic scratchpads cannot break the globality}
Next, we put forward a conjecture that agnostic scratchpads (scratchpads without direct supervision on the scratchpad tokens) cannot break the globality barrier. See Appendix \ref{app:intuition-agnostic-scratchpad} for further discussion. 

\begin{conjecture}\label{conjeture-agnostic}
Consider training a $T$-regular net with an $s$-scratchpad to learn $P_{X,Y}$ on a constant-sized alphabet by means of the following SGD algorithm. At each timestep, we draw a random sample $(X,Y)$ and compute a value for the net with scratchpad on $X$. Let $\mathcal{S}$ be the resulting scratchpad and for each $i\le s$ and each $\sigma \ne \mathcal{S}_i$, define an alternative scratchpad by setting the first $i-1$ entries of this scratchpad equal to those of $\mathcal{S}$, setting its $i$-th entry to $\sigma$, and using the net to compute the rest of its values. Then, regard the loss associated with setting the $i$th entry of the scratchpad to $\sigma$ as the loss resulting from the associated scratchpad, and use the resulting gradient to carry out one step of SGD. Then if $P_X$ is a well-behaved probability distribution, $P_{X,Y}$ is efficiently weakly learnable by a T-regular neural network with a scratchpad if and only if $P_{X,Y}$ has constant globality. 
\end{conjecture}
A natural counterpart of Theorem \ref{3cycleTheorem} holds for the previous conjecture (see Theorem \ref{agnostic3cycleTheorem}). In order to define the Transformer's loss on any given input it takes the expectation over every possible value of the scratchpad it might generate, and its proof is essentially identical to that of Theorem \ref{3cycleTheorem}.

\section{Scratchpads to break the globality}\label{sec:scratchpad}
{\bf Prior literature.} It has been shown that training Transformers with the intermediate steps required to solve a problem can enhance learning. This idea is usually referred to as providing models with a scratchpad \cite{nye2021work}. %For instance, consider math problems with a numeric final answer such as \cite{}. One approach for solving such questions is to train a model on pairs of question's text and its final answer. The scratchpad approach, on the other hand, trains the model with both the final answer and also a step-by-step solution, i.e., the model is trained to first generate a solution and then the final answer. 
The improved performance due to scratchpads has been reported on a variety of tasks including mathematical tasks and programming state evaluation \cite{nye2021work,anil2022exploring-length,bubeck}. See Appendix \ref{app:related-lit} for further references.

\subsection{Educated scratchpad}
% \begin{wrapfigure}{r}{0.45\textwidth}
% \centering
% \includegraphics[width=.99\linewidth]{matterhorn2.pdf}
%     \caption{An illustration showing how scratchpads can break the globality. The target may be efficiently learned if each scratchpad step is of low globality given the previous ones.}
%     \label{fig:scratchpad-globality-illustration}
% \end{wrapfigure}
We now provide a quantitative understanding of how the scratchpad can help with the notion of globality in the autoregressive setting (Definition \ref{def:seq-globality}). Assume that we want to learn target $Y\in \mathcal{A}$ from input $X \in \mathcal{A}^n$ such that $(X, Y) \sim D$ and $\mathrm{glob}(D)$ is high. If one can design intermediate targets $Y_1, \ldots, Y_k \in \mathcal{A}$ such that $Y_k = Y$ and the sequence $(X,Y_{<i}) \to Y_i$ has low globality according to Definition \ref{def:seq-globality}, then one can expect to learn each step of the sequence efficiently and thus the target at the end. In this case, the intermediate targets give the `educated scratchpad' (see Figure \ref{fig:scratchpad-globality-illustration} for an illustration).
We now show how designing low-globality scratchpads can help with learning by focusing on two examples: parity functions and the cycle task. 

\begin{figure}[t]
\centering
\includegraphics[width=.5\textwidth]{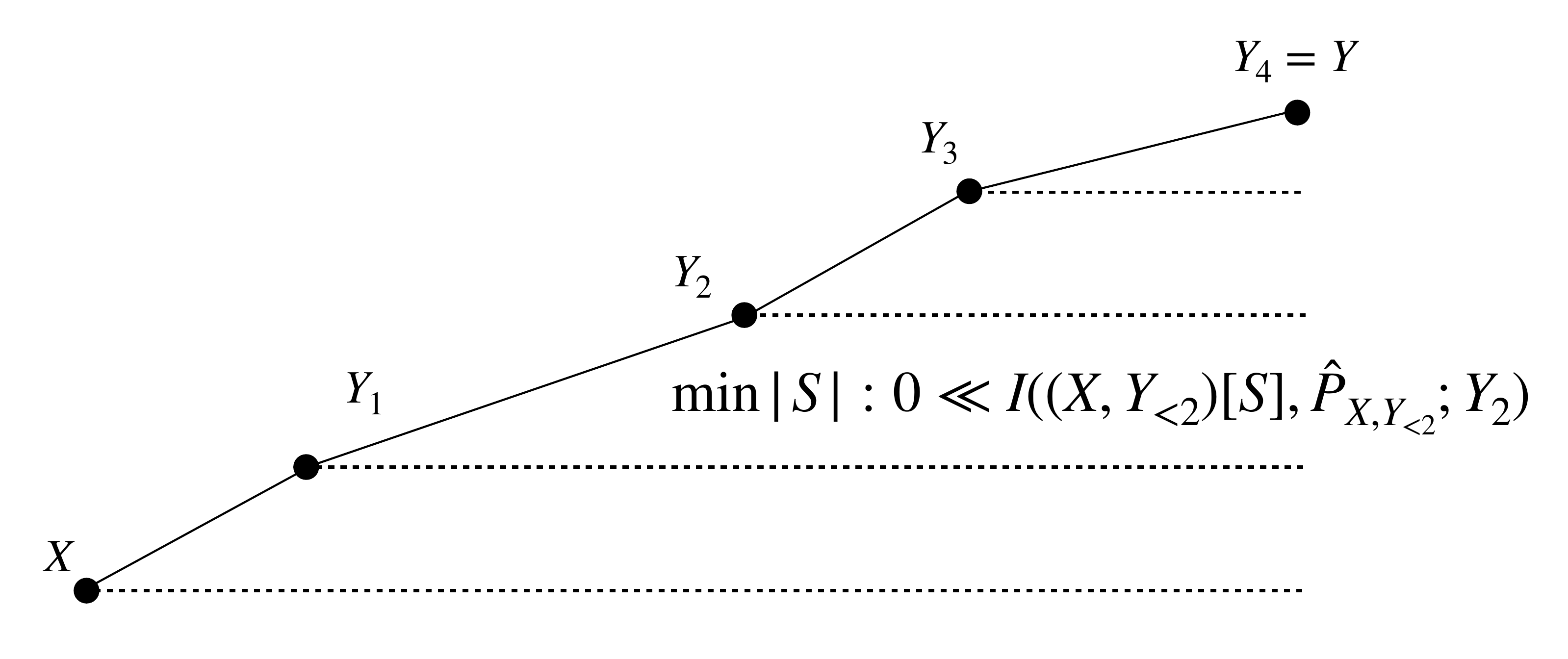}
    \caption{An illustration showing how scratchpads can break the globality. The target may be efficiently learned if each scratchpad step is of low globality given the previous ones.}
    \label{fig:scratchpad-globality-illustration}
\end{figure}

\paragraph{Results for learning parities.}\label{sec:sp-parity}
Consider learning parity function $y= f(x_1, \ldots, x_n) = x_1x_2\cdots x_k$ where $x_1, \ldots, x_n$ are drawn i.i.d.\ in $\{ \pm 1\}$ with uniform distribution.  For $k \leq \frac{n}{2}$, one can easily check that the globality of this task is $k$ as any $k-1$ coordinates are independent of the output and the histogram of the tokens does not help. Parity functions are known to be hard to learn \cite{abbe2020poly}. More specifically, it has been previously shown that as $k$ increases the parity task becomes harder to learn to the point that parity of degree $\min\{k, n-k\}=\omega(1)$ cannot be learned in $poly(n)$ time with standard poly-size neural networks under standard training assumptions \cite{abbe2020poly}. Note that this is consistent with our results, as the globality is non-constant.

Now, consider learning this task with a scratchpad that breaks down the learning with intermediate targets $y_1, y_2, \ldots, y_k$ such that 
\begin{equation*}
    y_1=x_1, \qquad y_2=x_1x_2,  \qquad \ldots \quad  y_i = y_{i-1}x_i, \quad \ldots \quad y_k =x_1x_2\cdots  x_k =f(x), 
\end{equation*} 
i.e.,  $y_i$ is the cumulative product of the first $i$ bits. Note that each intermediate target $y_i$ can be computed by using at most 2 of the previous tokens, implying the following lemma.
\begin{lemma}
The parity task with the cumulative product scratchpad has a globality of $2$.
\end{lemma}
% In the same manner, one can construct scratchpads with arbitrary jumps of size $j$ as in
% \begin{equation*}
    % y_1=x_1x_2\cdots x_{j},\qquad \dots,\qquad y_i = y_{i-1}\prod_{k=(i-1)j+1}^{ij}x_{k},\qquad \dots, y_{\lceil \nicefrac{k}{j} \rceil} =x_1x_2\cdots x_k =f(x).
% \end{equation*} 
% Similarly, the globality of the parity task with a scratchpad of jump $j$ is $j+1$ since any intermediate token can be computed using at most $j+1$ tokens.
Transformers with such a scratchpad can in fact easily learn parity targets, see Appendix \ref{sec:parities-scratchpad}.

\paragraph{Results for the cycle task.} Consider the cycle task and a scratchpad that learns the depth-first search (DFS) algorithm from the source query node.\footnote{For the particular graph structure of the cycle task, DFS is the same as the breadth-first search (BFS).} For example, consider the following input corresponding to two cycles $a,x,q$ and $n,y,t$:
\texttt{a>x; n>y; q>a; t>n; y>t; x>q; a?t;}.
In this case, doing a DFS from node \texttt{a} gives \texttt{a>x>q>a} where the fact that we have returned to the source node \texttt{a} and not seen the destination \texttt{t} indicates that the two nodes are not connected. Therefore, the full scratchpad with the final answer can be designed as \texttt{a>x>q>a;0}. Similarly, if the two nodes were connected the scratchpad would be \texttt{a>...>t;1}. One can easily check that the cycle task becomes low-globality with the DFS scratchpad. 

\begin{lemma}
The cycle task with the DFS scratchpad has a globality of $3$.
\end{lemma}
This follows from the fact that one only needs to find the next node in the DFS path (besides the label), which one can check with polynomial chance by checking the first edge. 

In Figure \ref{fig:cycles-scratchpad} we show that a decoder-only Transformer with the DFS scratchpad in fact learns the cycle task when $n$ scales.

% \begin{figure}
%     \centering\includegraphics[width=0.6\textwidth]{figures/plot3.pdf}
%     \caption{The cycle task is learned easily when the DFS/BFS scratchpad is used.}
%     \label{fig:cycles-scratchpad}
% \end{figure}

% \begin{figure}
%     \centering\includegraphics[width=.9\textwidth]{figures/num-iter-vs-cycle-size-model-per-col.png}
%     \caption{The cycle task is learned easily when the DFS scratchpad is used.}
%     \label{fig:cycles-scratchpad}
% \end{figure}

\begin{figure}
     \centering
     \begin{subfigure}[b]{0.45\textwidth}
         \centering
         \includegraphics[width=\textwidth]{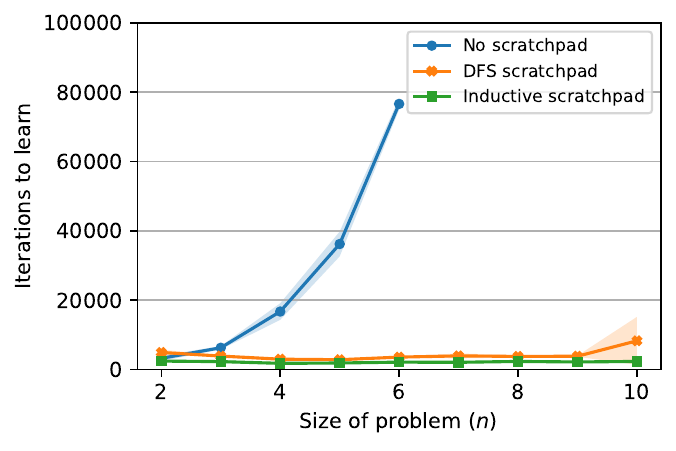}
         \caption{The cycle task is learned easily when the DFS scratchpad is used.}
         \label{fig:cycles-scratchpad}
     \end{subfigure}
     \hfill
     \begin{subfigure}[b]{0.45\textwidth}
         \centering
         \includegraphics[width=\textwidth]{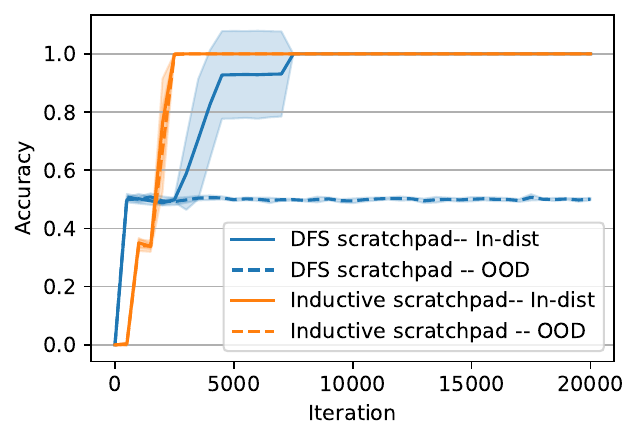}
         \caption{The DFS scratchpad fails to generalize to OOD samples while the inductive scratchpad can.}
         \label{fig:inductive-sp-learns}
    \end{subfigure}
    \caption{(Left) Learning the cycle task with a scratchpad. (Right) OOD generalization for the DFS and inductive scratchpads (see Section \ref{sec:full-fails}).}  
\end{figure}

\begin{remark}
If one has full knowledge of the target function, one could break the target into sub-targets using an educated scratchpad to keep the globality low and thus learn more efficiently (of course one does not have to learn the target under full target knowledge, but one may still want to let a model learn it to develop useful representations in a broader/meta-learning context). One could in theory push this to learning any target that is poly-time computable by emulating a Turing machine in the steps of the scratchpad to keep the overall globality low. Some works have derived results in that direction, such as \cite{malach} for some type of linear autoregressive models, or \cite{abbe2020poly} for more abstract neural nets that emulate any Turing machine with SGD training. However, these are mostly theory-oriented works. In practice, one may be instead interested in devising a more `generic' scratchpad. In particular, a relevant feature in many reasoning tasks is the power of induction. For instance, the parity and cycle tasks are two examples where learning an induction step function appears useful.    
\end{remark}

\subsection{Inductive Scratchpads}
As discussed, scratchpads can break the global reasoning barrier with appropriate mid-steps. In this part, however, we show that fully educated scratchpads can be sensitive to the number of reasoning steps, translating into poor out-of-distribution (OOD) generalization. As a remedy, we put forward the concept of inductive scratchpad which applies to various reasoning tasks as in previous sections. 
\subsubsection{Educated scratchpad can overfit in-distribution samples}\label{sec:full-fails}
Consider the cycle task with $24$ nodes. For the test distribution, we use the normal version of the cycle task, i.e., either two cycles of size $12$ and the nodes are not connected or a single cycle of size $24$ where the distance between the query nodes is $12$. For the train distribution, we keep the same number of nodes and edges (so the model does not need to rely on new positional embeddings for the input) but break the cycles to have uneven lengths: (1) a cycle of size $6$ and a cycle of size $18$ when the two nodes are not connected (the source query node is always in the cycle of size $6$) or (2) a cycle of size $24$ where the nodes are at distance $6$. Thus, in general, we always have $24$ nodes/edges in the graphs. However, the length of the DFS path (i.e., number of reasoning steps) is $6$ at training and $12$ at test. We trained our model on this version of the task with the DFS scratchpad. The results are shown in Figure \ref{fig:inductive-sp-learns}. We observe that the model quickly achieves perfect accuracy on the training distribution, yet, it fails to generalize OOD as the model overfits the scratchpad length and number of reasoning steps. %Note that in the DFS scratchpad, an identical operation (finding the next node) is repeated a number of times to generate the whole scratchpad. In other words, the model needs to continue applying the same operation to generalize to OOD examples. 
In the next part, we introduce the notion of inductive scratchpad to fix this problem.

\subsubsection{Inductive scratchpad: definition and experimental results}
In a large class of reasoning tasks, one can iteratively apply an operation to some state variable (e.g., a state array) to compute the output. This applies in particular to numerous graph algorithms (e.g., shortest path algorithms such as BFS or Dijkstra's algorithm), optimization algorithms (such as genetic algorithms or gradient descent), and arithmetic tasks.
%We formalize them with the abstract notation below. 

\begin{definition}[Inductive tasks]
Let $Q$ be the question (input). We say that a task can be solved inductively when there is an induction function (or a state transition function) $g$ such that  
\begin{align*}
    s[1]  = g(Q, \emptyset), 
    \quad 
    s[2] = g(Q, s[1]), \quad 
%    s[3] &= g(Q, s[2])
\ldots , \quad
    s[k] &= g(Q, s[k-1]),
\end{align*}
where $s[1], \ldots, s[k]$ are the steps (or states) that are computed inductively. For example, the steps/states could be an array or the state of an automata that is being updated. Note that the termination is determined by the state. In the context of Transformers, one can use the generation of the end of sequence token \texttt{<EOS>} to terminate.  
\end{definition}

\paragraph{Inductive tasks with a fully educated scratchpad can overfit proofs.}
The fully educated scratchpad for the question $Q$ as input would be $\texttt{s[1];s[2];...;s[k]<EOS>}$, where the token \texttt{<EOS>} ends the generation. However, this method may not fully utilize the fact that each state is only generated from the last state by applying the same (set of) operation(s). In particular, \texttt{s[k]} typically attends to all of the previous states. Further, the model may not be able to increase the number of induction steps beyond what it has seen during training, as shown in Figure \ref{fig:inductive-sp-learns} for the cycle task. 

Now we show that by using attention masking and reindexing the positions of the tokens, one can promote the desired `inductive' behavior. We call this the inductive scratchpad. As three showcases, we demonstrate that the inductive scratchpad can improve OOD generalization on the cycle task and length generalization on parity and addition tasks. 

\paragraph{Inductive scratchpad implementation.} The inductive scratchpad for an inductive task is similar in format to the fully educated scratchpad but it has the following modifications: 
(1) tokens: two new special tokens are used: the \texttt{<START>} token which separates the question from the intermediate states and the~\texttt{<STATE>}~token (denoted \texttt{\#} hereafter) to separate the states. Using these tokens, for an input question $Q$, the format of the inductive scratchpad reads $\texttt{<START>s[1]\#s[2]\#...\#s[k]<EOS>}$. 
(2) generation: we want the model to promote induction and thus `forget' all the previous states except the last one for the new state update. I.e., we want to generate tokens of \texttt{s[i+1]} as if the input was \texttt{Q<START>s[i]\#}. To implement this, one can use attention masking and reindex positions (in order to have a proper induction) or simply remove the previous states at each time; 
(3) training: when training the scratchpad, we want the model to learn the induction function $g$, i.e., learning how to output $\texttt{s[i+1]\#}$ from $\texttt{Q<START>s[i]\#}$, which can be achieved with attention masking and reindexing the positions. As a result, the inductive scratchpad can be easily integrated with the common language models without changing their behavior on other tasks/data.
We refer to Appendix \ref{app:inductive-sp-implement} for a detailed description of the inductive scratchpad implementation. 

\begin{comment}
Inductive behavior for Transformers can be described as follows: At train, the input question and the inductive scratchpad sequence \texttt{Q<START>s[1]\#s[2]\#...\#s[k]<EOS>} should be equivalent to seeing following samples:
\begin{center}
    \texttt{Q<START>s[1]\#}, \dots, \texttt{Q<START>s[i]\#s[i+1]\#}, \dots, \texttt{Q<START>s[k-1]\#s[k]<EOS>}.
\end{center} Similarly, the generation of tokens of \texttt{state[i+1]} given \texttt{Q<START>s[1]\#...\#s[i]<EOS>} should be equivalent to the generation process as if the sequence to this point was \texttt{Q<START>s[i]\#}. 
Using special tokens of \texttt{<START>} and \texttt{<STATE>}, the inductive behavior is easy to implement using attention masking (e.g., \texttt{s[i]} only attends to anything which has come before the \texttt{<START>} token, particularly the question $Q$, and the previous state \texttt{s[i-1]}) and reindexing the positions of the tokens. We will explain the implementation in Appendix \ref{app:inductive-sp-implement} in more detail. 
Note that scratchpad can also contain some tokens before \texttt{<START>}, that can be used for the generation of tokens of each state. In that sense, one can view everything before \texttt{<START>} as a permanent memory and the states as a temporary memory that gets updated. Also, note that the behavior of the Transformer only changes in the presence of \texttt{<START>} and \texttt{<\#>} tokens, thus, the model's behavior on non-inductive data remains unchanged which shows that this approach can be easily integrated with different models. 
\end{comment}

\paragraph{Inductive scratchpad for the cycle task.} The DFS scratchpad of the cycle task can be made inductive by making each step of the DFS algorithm a state. E.g., for the input
\texttt{a>x;n>y;q>a;t>n;y>t; x>q;a?t;}, the DFS scratchpad is \texttt{a>x>q>a;0<EOS>}, and the inductive scratchpad becomes \texttt{<START>a\#x\#q\#a;0<EOS>} where each state tracks the current node in the DFS. 
In Figure \ref{fig:inductive-sp-learns}, we show that the inductive scratchpad for the cycle task can generalize to more reasoning steps than what is seen during training, and thus generalize OOD when the distance between the nodes is increased. 

\paragraph{Length generalization for parity and addition tasks.} We can use inductive scratchpads to achieve length generalization for the parity and addition tasks. For parities, we insert random spaces between the bits and design an inductive scratchpad based on the position of the bits and then compute the parity iteratively. The performance of this inductive scratchpad is depicted in Figure \ref{fig:parity-length-gen} where we can see a Transformer trained on inputs with up to $30$ bits can generalize to samples with up to $50/55$ bits depending on the seed.  For the addition task, we propose two inductive scratchpads. (1) \textit{Random space method} that requires random spaces between the digits in the input and uses the position of the digits to compute the addition digit-by-digit (similar to the parity). With this scratchpad, we can generalize to numbers with 18 digits while training on numbers with up to 10 digits. (2) \textit{Shift method} that uses random tokens in the input and computes the addition digit-by-digit by shifting the operands. The latter enables us to generalize from 4 to 26 digits at the cost of having a less natural input formatting. The results for different seeds are provided in Figure \ref{fig:length-gen-addition}. See details of these scratchpads in Appendices~\ref{app:parity-length-gen},~\ref{app:addition-length-gen}.\footnote{Our code is available at \url{https://github.com/aryol/inductive-scratchpad}.} A rough comparison between the performance of different methods for addition is given in Table \ref{tab:addition-comparison-main}. Note that the settings used in these works are not exactly the same, e.g., our methods often work with smaller models and more natural input formatting. See Appendix \ref{app:len-gen-parity-addition} for a detailed comparison for both the parity and addition tasks.

\begin{table}[h]
\vspace{-3mm}
    \centering
    \caption{Length generalization of different methods for the addition task where our methods are shown in bold. $a \rightarrow b$ means generalizing to $b$ digits when trained on $a$ digits.}
    \vspace{1mm}
    \label{tab:addition-comparison-main}
    \small
    \addtolength{\tabcolsep}{-2.5pt}
    \begin{tabular}{|c|c|c|c|c|c|c|c|}
    \hline
        \textbf{Method} & \cite{kazemnejad2023impact} & \cite{bubeck} & \cite{zhou2023algorithms} & \cite{jelassi2023length} & \cite{zhou2024transformers} & \textbf{random space method} & \textbf{shift method} \\\hline
        \textbf{Performance} & $8\rightarrow 9$ & $10 \rightarrow 12$ & $40 \rightarrow 50$  & $5 \rightarrow 15$ & $40 \rightarrow 65$ & $10 \rightarrow 18$ & $4 \rightarrow 26$ \\\hline
    \end{tabular}
    \addtolength{\tabcolsep}{2.5pt}
    \vspace{-1mm}
\end{table}

\begin{figure}[t]
     \centering
     \begin{subfigure}[t]{0.45\textwidth}
         \centering
         \includegraphics[width=\textwidth]{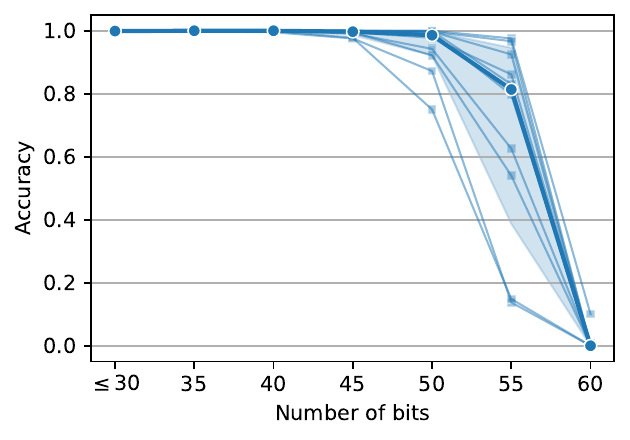}
         \caption{Length generalization for parity when training is done up to $30$ bits.}
         \label{fig:parity-length-gen}
     \end{subfigure}
     \hfill
    \begin{subfigure}[t]{0.45\textwidth}
         \centering
         \includegraphics[width=\textwidth]{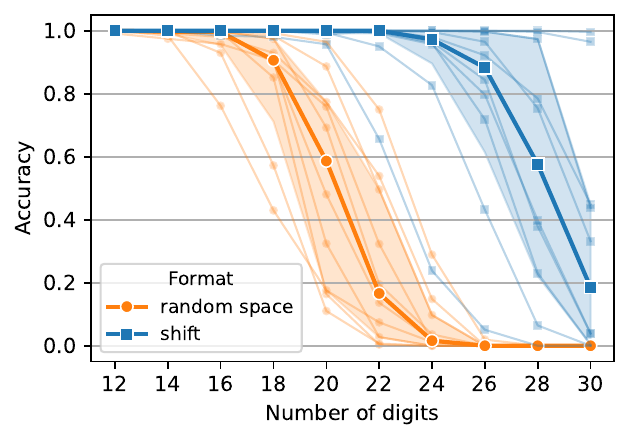}
         \caption{Length generalization for addition where the random space method is trained up to $10$ digits and the shift method is trained up to $4$ digits.}
         \label{fig:length-gen-addition}
     \end{subfigure}
    \caption{Length generalization for parity and addition tasks using different random seeds. The medians of the results are highlighted in bold.}  
\vspace{-4mm}
\end{figure}

\section{Conclusion}
This paper shows that for the learning objective and in contrast to expressivity results, Transformers trained from scratch have a `global reasoning barrier' quantified by the globality degree.  
The globality measure has a simpler form and broader applicability range than prior measures as discussed in Appendix \ref{lower_bound_ref}, it also has tighter applicability for Transformers. The measure is currently defined for weak learning (inverse-polynomial or constant edge), and a natural next step is to consider stronger learning requirements with notions of `globality leap', e.g., expanding the current work in the direction of \cite{abbe2023leap} but for more general distributions. Investigating the role of curriculum learning is another natural direction and we provide preliminary results here in Appendix~\ref{sec:curriculum-learning}.

%Also, learning the full parity with Transformers may be unstable to noisy gradient descent, but it remains achievable for some choices of the hyperparameters. 

The globality is also defined in the autoregressive setting, to better quantify when scratchpads can break targets into easier sub-targets. Two negative results are shown for scratchpads: agnostic scratchpads still suffer from the globality barrier, and fully educated scratchpads can have poor OOD generalization. This motivates the introduction of the inductive scratchpad. 

The inductive scratchpad can be used for a broad range of reasoning/algorithmic tasks and can easily be integrated into Transformers. The inductive scratchpad is independent of the number of reasoning steps since the model only learns the induction function. Consequently, the model better generalizes to inputs requiring different numbers of reasoning steps. This gives improvements of OOD/length generalization for the cycle task (Figure \ref{fig:inductive-sp-learns}), parity (Figure \ref{fig:parity-length-gen}), and addition (Figure \ref{fig:length-gen-addition}). 

%We also note that, although Transformers may learn algorithmic tasks using scratchpads, executing an algorithm using a programming language on a CPU would often be less time-consuming than generating the corresponding scratchpad tokens using Transformers. So in practice, we believe that a joint approach may be beneficial: the model can do part of the computation/reasoning via learning and token generation, while subroutine may be executed by calling programs such as calculators or other engines. 

Another interesting aspect is whether the model can use an inductive behavior on new tasks if it was pre-trained on prior inductive tasks. Note that the inductive behavior of the inductive scratchpad is only determined by two special tokens. Thus, in principle, models can generate these special tokens and go into the inductive state for other tasks if pre-trained on inductive data. We leave the general investigations of pre-trained models and the automated learning of more general scratchpads, capitalizing on the measures defined here, to future works.

\section*{Acknowledgments}
We thank Samira Abnar, Tatiana Likhomanenko, Joshua Susskind, and Russ Webb for stimulating discussions and useful feedback on the research of this paper, as well as Etai Littwin and Preetum Nakkiran for giving feedback on the paper's writing.

\bibliography{references}
\bibliographystyle{unsrt}
%%%%%%%%%%%%%%%%%%%%%%%%%%%%%%%%%%%%%%%%%%%%%%%%%%%%%%%%%%%%
\newpage
\appendix

\section{Further related literature}\label{app:related-lit}
\subsection{Reasoning capabilities of Transformers} 
\paragraph{Reasoning vs. memorization.}
The performance of large language models and Transformers has been shown to improve as the model size, amount of data, and training time are increased \cite{kaplan2020scaling-laws, alabdulmohsin2022revisiting}. Furthermore, it has been shown that Transformers are prone to memorizing the training data \cite{zhang2016understanding, feldman2020neural-mem2, carlini2022quantifying-mem4, kandpal2022deduplicating-mem3, mahdavi2024memorization}. Thus, it is natural to ask whether Transformers mostly rely on memorization or if they in fact use memorization along with a significant degree of reasoning. Reasoning is also essential in solving more challenging tasks such as planning and mathematics. In recent years, the reasoning power of Transformers has been studied on synthetic reasoning tasks such as PVR \cite{Zhang2021PointerVR}, LEGO \cite{zhang2022unveiling}, algorithmic tasks such as CLRS \cite{velivckovic2022clrs}, and more natural settings such as solving math problems \cite{saxton2019analysing, lewkowycz2022solving-minerva}. Note that reasoning tasks often have a combinatorial nature and thus an exponentially large input space. Moreover, this input space may not lie on a low-dimensional manifold which makes memorization approaches ineffective. For example, in arithmetic, all digit combinations are often possible and also change the result significantly, whereas, in text, only specific combinations of words are valid and besides changing a single word often does not change the meaning of the text drastically. Another important criterion for reasoning is the ability to generalize to out-of-distribution (OOD) samples and in-particular length generalization \cite{zaremba2014learning, lake2018generalization, hupkes2020compositionality, anil2022exploring-length, icml-version}. It has been observed that simple tasks such as addition, multiplication, and parity are generally hard for Transformers both in the in-distribution setting and more notably in the length generalization setting \cite{anil2022exploring-length, lee2024teaching-arithmetic}.

\paragraph{Positional embeddings.}
An important component of length generalization is the positional embedding of the model. It has been shown that various forms of relative positional embedding \cite{Shaw2018SelfAttentionWR, Dai2019TransformerXLAL} including T5's relative positional embedding \cite{raffel2019exploring}, ALiBi \cite{alibi}, and Rotary \cite{rotary} can yield better length generalization performance than absolute positional embedding \cite{csordas2021devil, ontanon-etal-2022-length-generalization, kazemnejad2023impact}. In particular, the recent work of Kazemnejad et al. \cite{kazemnejad2023impact} evaluates different positional embeddings for decoder-only Transformers on a benchmark of reasoning tasks where it is shown that relative positional embeddings perform better than absolute ones. Interestingly, it is also shown that decoder-only Transformers with no positional embedding (NoPE) can still encode positional information and in fact have better length generalization performance than absolute/relative positional embeddings on certain tasks. The inductive scratchpad put forward in this paper also reindexes the position of the tokens for each new state. Using this technique, the inductive scratchpad circumvents the need for learning new positional information as the number of reasoning steps (i.e., the number of states) increases. Nevertheless, the inductive scratchpad can be used with both absolute and relative positional embeddings. We further discuss the relation between the inductive scratchpad and relative positional embeddings in Appendix~\ref{app:indutive-relative-PE}.

\paragraph{Architectural modifications.} Several architectural modifications have also been proposed that can potentially enhance the reasoning ability of Transformers on certain tasks. A line of work focuses on adding recurrent structures to Transformers \cite{dehghani2019universal, hutchins2022-block-recurrent, giannou23a-looped}. In particular, Universal Transformers \cite{dehghani2019universal}, share the weights between Transformer layers and also have a variable depth (similar to adaptive computation time \cite{graves2016adaptive}) implemented using a dynamic halting mechanism. Generally scratchpads and in particular the inductive scratchpad also share some recurrent flavor since when each token of the scratchpad is generated, the whole input and the scratchpad tokens are given to the Transformer again. The differences are however quite significant both on the side of supervision (e.g., scratchpads are usually supervised) and halting mechanism (e.g., generation of \texttt{<EOS>} token ends the process for scratchpad models). Another relevant architecture is the neural data router \cite{csordas2022-neural-data-router-ndr} where the weights are shared among transformer layers and also copying gates and a special attention format, geometric attention, are used. The copying gate allows some tokens not to be transformed and instead just be copied at certain transformer layers while geometric attention induces a recency bias. The neural data router is shown to generalize to samples with up to three more operations than what is seen during training for simple arithmetic tasks and ListOps dataset \cite{Nangia2018ListOpsAD}. We note that our inductive scratchpad can also generalize to samples with a higher number of operations/reasoning steps. Note that the neural data router approach requires the model's depth to be (at least) equal to the maximum number of operations, while our model does not have such limitations and can thus be trained and tested on samples with larger numbers of operations than the neural data router approach. Moreover, intermediate steps are supervised through the scratchpad mechanism in our approach while there is no supervision for intermediate steps in the neural data router.

\paragraph{Scratchpad and chain-of-thought.} Nye et al. \cite{nye2021work} put forward the idea of using scratchpads by showing that training Transformers on the intermediate steps in addition to the final answer can improve their performance on different tasks such as executing Python code, addition, and evaluating polynomials. 
Similarly, in chain-of-though reasoning (CoT) \cite{wei2023chainofthought} models are shown step-by-step demos of problem-solving (or models generate chains of thought without any examples as in zero-shot CoT \cite{kojima2023large}, among other variants). It has been further shown that using explicit explanations and reducing ambiguity can be proven useful as in the notion of algorithmic prompting \cite{zhou2022teaching}. 
Lanchantin et al. \cite{lanchantin2024self-note} introduced the concept of self-notes, showing that interleaving the intermediate reasoning steps within the question/context rather than after the question can boost the performance of scratchpad/CoT on certain tasks. Goyal et al. \cite{goyal2024pause-tokens} introduced pause tokens which act as dummy tokens that provide models with more compute time and processing before the generation of the true next token. On a related note, the iterative prompting method in \cite{sun2024enhancing} involves querying LLMs iteratively to optimize question prompts. Further, the Select and Inference framework (SI) \cite{creswell2022selectioninference} utilizes pre-trained LLMs as general processing modules, alternating between selection and inference to generate interpretable reasoning steps leading to the final results.

The work of \cite{feng2024towards} studies the role of scratchpad from an expressivity point of view. Assuming $TC^0 \neq NC^1$, they show the existence of tasks that are not expressible by constant-depth Transformers without scratchpad and expressible by constant-depth Transformers with scratchpad. It is further shown that dynamic programming (DP) algorithms can be expressed by constant-depth Transformers with scratchpad \cite{feng2024towards}. 
In this work, however, we used the concept of autoregressive globality to explain why scratchpads are helpful and how to design them from the learning point of view. Further, we introduced inductive scratchpads that are suitable for a broad range of algorithmic reasoning tasks and can potentially generalize to more complex samples than what appears in the train distribution. We note that inductive scratchpads are also suitable for DP algorithms thanks to their iterative structure (the variables of the DP algorithm can be updated in each state of the inductive scratchpad). We stress that our inductive scratchpad potentially enables models to generalize to OOD samples for such DP algorithms and/or to work with longer scratchpads because of the attention masking which reduces the effective context size. In concurrent research, \cite{cabannes2024iteration} shows that a restricted class of iterative algorithms with scratchpad (including parity task) are expressible by 2-layer Transformers. Compared to our inductive scratchpad, \cite{cabannes2024iteration} does not have any length generalization result.

\subsection{Length generalization for parity and addition tasks} \label{app:len-gen-parity-addition}
Parity and addition tasks are two essential reasoning tasks that are challenging in the setting of length generalization where the number of bits/digits is increased. These tasks can also be hard in the in-distribution setting if the number of bits/digits is large enough.
It has been shown that certain forms of scratchpad/chain-of-though reasoning and relative positional embedding can achieve a modest length generalization on the parity task (around 5/10 more bits depending on the method and the seed) \cite{anil2022exploring-length,kazemnejad2023impact}. The RASP-L work \cite{zhou2023algorithms} considers the parity task and can get length generalization from 30 to 50 bits (depending on the random seed) by using `index hints' which are special tokens that come before the input bits in a specific order. For instance, an input could look like $a0b0c1d1e0$. Our proposed inductive scratchpad for the parity task only requires the use of random spaces in the question formatting and can generalize to 55 bits when trained on samples with up to 30 bits. Different works use different input formatting, techniques, model sizes, and train/test datasets, nevertheless, a rough comparison between different works for the parity task is given in Table \ref{tab:parity-app-len-gen} (learning is defined as above $80\%$ accuracy for the majority/median of the results for different seeds, some seeds do better than others).

\begin{table}[h!]
\centering
\small % Use \small to make the table font smaller
\caption{Length generalization performance of different methods for the parity task}
\label{tab:parity-app-len-gen}
\begin{tabular}{lll}
\toprule
\textbf{Work} & \textbf{Performance} & \textbf{Method and Assumptions} \\
\midrule
Kazemnejad et al. \cite{kazemnejad2023impact} & From 8 bits to 12 bits & \makecell[l]{Using NoPE (no positional embedding)} \\
\midrule
RASP-L work \cite{zhou2023algorithms} & From 30 bits to 50 bits & \makecell[l]{Using scratchpad + ‘index hints’ \\ (special tokens before each bit in the input),\\
an input \& output look like $a0b0c1d1e0 > +c-d+$} \\
\midrule
Anil et al. \cite{anil2022exploring-length} & From 8 bits to 20 bits & \makecell[l]{Using large pre-trained models (128B) \\ + prompting + fine-tuning + scratchpad} \\
\midrule
\textbf{Our method} & From 30 bits to 55 bits & \makecell[l]{Using random spaces in the input \\(e.g., \_01\_10\_0\_\_1\_) + inductive scratchpad} \\
\bottomrule
\end{tabular}
\end{table}

For the addition task, Lee et al. \cite{lee2024teaching-arithmetic} use a decoder-only model similar to ours (also similar in size) showing that generating the output in the reverse format and using scratchpads are helpful. However, they are not able to get length generalization with their model size. Nevertheless, it has been shown that large enough models (with more than $10^8$ parameters) with a scratchpad can generalize to numbers with 9/10 digits while being trained on numbers with up to 8 digits \cite{nye2021work}.
Jelassi et al. \cite{jelassi2023length} show that encoder-only models with relative positional embedding can generalize to numbers with 15 digits when they are trained on numbers with up to 5 digits. The RASP-L work \cite{zhou2023algorithms} considers the addition task and can get length generalization from 35/40 to 50 digits (depending on the seed) by outputting the result in the reverse order and also using `index hints' (special tokens coming before the operands' digits in a specific order). For instance, an example would look like $a5b4+a3b7=b1a9$. Zhou et al. \cite{zhou2024transformers} further improve the latter by using FIRE relative positional embedding \cite{li2023functional-FIRE} and randomized position encoding \cite{ruoss2023randomized-randomized-pos-enc}, generalizing to numbers with 65 digits when trained on numbers with up to 40 digits. 
The work of Shen et al. \cite{bubeck} uses a scratchpad with a recursive format for the addition task to generalize to numbers with 12/13 digits while the training data includes numbers with up to 10 digits. Our special case of inductive scratchpad based on shifting the numbers is similar to their recursive scratchpad, however, \cite{bubeck} does not enforce any inductive structure (as achieved with the attention masking for the inductive scratchpad) and hence \cite{bubeck} gets a much more limited length generalization.

In this work, we proposed two inductive scratchpads for the addition task. The inductive scratchpad that requires random spaces in the question can generalize to numbers with  18/20 digits when trained on numbers with up to 10 digits. The other inductive scratchpad based on shifting the operands at each step can generalize to numbers with up to 26 digits when trained on numbers with up to 4 digits at the cost of having a less natural question formatting. Different works use different input formatting, techniques, model sizes, train/test distribution, and evaluation procedures, nonetheless, a rough comparison between different works for the addition task is provided in Table \ref{tab:addition-app-len-gen} (reported results correspond to the median performance given by different seeds). We believe our solutions for both the parity and addition tasks require one of the least stringent modifications of the input and provide a significant improvement of the length generalization performance compared to prior works even though the models that we have used are often remarkably smaller.

\begin{table}[h!]
\centering
\small
\caption{Length generalization performance of different methods for the addition task}
\label{tab:addition-app-len-gen}
\begin{tabular}{lll}
\toprule
\textbf{Work} & \textbf{Performance} & \textbf{Method and Assumptions} \\
\midrule
Kazemnejad et al. \cite{kazemnejad2023impact} & From 8 digits to 9 digits & \makecell[l]{Using NoPE (no positional embedding)} \\
\midrule
Shen et al. \cite{bubeck} & From 10 digits to 12 digits & \makecell[l]{Scratchpad with recursive format} \\
\midrule
RASP-L work \cite{zhou2023algorithms} & From 40 digits to 50 digits & \makecell[l]{Reverse order of the output + ‘index hints’ \\ (special tokens before each digit),\\ e.g., $a5b4 + a3b7 = b1a9$} \\
\midrule
Jelassi et al. \cite{jelassi2023length} & From 5 digits to 15 digits & \makecell[l]{Encoder-only model \\+ relative pos. emb. + padded inputs} \\
\midrule
Zhou et al. \cite{zhou2024transformers} & From 40 digits to 65 digits & \makecell[l]{FIRE relative pos. emb. + randomized \\ position encodings + reversed output \\+ index hints} \\
\midrule
\textbf{Our random space method} & From 10 to 18 digits & \makecell[l]{Inductive scratchpad +  random space in \\ the input (e.g., \(94\_+\_3\_\_1=\))} \\
\midrule
\textbf{Our shift method} & From 4 to 26 digits & \makecell[l]{Inductive scratchpad + random text before \\ each operand (e.g., \(fs\$46+ih\$98\))} \\
\bottomrule
\end{tabular}
\end{table}

\subsection{Reasoning over graphs}
Numerous reasoning tasks can be thought of as some form of rule-based logical inference or entailments, over implicit or explicit graphs, where at their core sit common graph operations such as graph connectivity checks  (as in the cycle task in the present paper). Here, we review a sample of notable related works on logical reasoning over graphs sharing the same primary motivation as the one behind our current work.

In recent years, several benchmarks have emerged, focusing on logical inference over diverse modalities. Within the context of natural language, LogiQA \cite{liu2020logiqa}, DROP \cite{dua2019drop} or the Cluttr benchmarking set \cite{sinha2019clutrr} are worth mentioning. These benchmarks assess logical relational reasoning abilities concerning entities and relations expressed in natural language. Another example but from a different modality is STAR \cite{wu2021star_situated_reasoning}, which focuses on situated logical reasoning over video clips. The work by Ontanon et al. \cite{ontanon2022logicinference} introduced the LogicInference dataset for training and benchmarking sequence-to-sequence models for logical inference tasks. The work by Wang et al. \cite{wang2024language} introduces NLGraph, a benchmarking set for measuring LLMs’ capabilities in solving graph-based problems, including the graph connectivity task studied in the current paper. GraphQA, as presented in the work by Fatemi et al. \cite{fatemi2023talk}, explores the impact of graph structure on language model prompting, while separating the impact of graph encoding and question prompting on the final performance. The common denominator of these studies' results is that while language models excel as shallow reasoners, their performance in logical reasoning deteriorates with an increase in the number of required reasoning steps~\cite{creswell2022selectioninference}.

On the modeling front, a group of efforts have primarily focused on baking the appropriate inductive biases into the models for reasoning and logical inference over graphs. The work  \cite{zhang2022greaselm} proposes architectural innovations to use both the context and external knowledge sources. \cite{cheng2023neural} proposes an end-to-end neural model for learning compositional logical rules. The Neurosymbolic approach, presented in \cite{Olausson_2023} involves performing semantic parsing on natural language input using neural components to transform the input into discrete structures suitable for symbolic reasoning by the logic theorem provers. Graphomer \cite{ying2021transformers} extends standard Transformer architecture with a graph reasoning inductive bias. The work of \cite{dehghani2019universal} combines Transformers with the inductive bias of the recurrent models. Other works take the route of generating more data using external sources like a Knowledge Graph~\cite{agarwal2021knowledge} or extending pre-training \cite{jin2023patton}.  
The work in \cite{wang2024language} evaluates LLMs' graph reasoning capabilities in the presence of various prompting techniques like scratchpads/CoT and their variants. This study introduces Build-a-Graph Prompting, an instruction-based prompting method, as well as Algorithmic Prompting, which includes references to relevant graph algorithms in the prompts, as two approaches for enhancing LLMs in solving graph tasks in natural language.

Concurrently with our work, \cite{sanford2024understanding} studies the graph reasoning abilities of Transformers. They focus on different tasks such as node count, edge count, connectivity, and shortest path. On the theoretical side, they focus on the expressivity of Transformers. In particular, they show that log-depth Transformers can express the graph connectivity task. This is while in our paper, we focus on learning, showing that for hard enough distributions (i.e., distributions with high globality), regular Transformers cannot learn the connectivity task despite being able to express it.  On the empirical side, they show that Transformers can perform well on the connectivity tasks in GraphQA dataset \cite{fatemi2023talk}, outperforming graph neural networks (GNNs). We note that this is not in contrast to our results. As we saw in Section \ref{hardness} with random graphs, if the input distribution has low globality the connectivity task can become learnable for Transformers.

\subsection{Learning measures and lower-bounds for GD}\label{lower_bound_ref}

Some recent literature has studied complexity measures for (S)GD-trained neural networks. In particular, the noise sensitivity \cite{o'donnell_2014,Zhang2021PointerVR,abbe2022learning, hahn2024sensitive}, which applies mostly to settings with i.i.d. inputs and is known to be loose for strong learning \cite{mergedstaircase,abbe2023leap}; the statistical query (SQ) dimension \cite{kearns,feldman_gen} and the cross-predictability \cite{abbe2020poly}, which are usually defined for a class of targets/distributions rather than a single distribution (in particular the full parity is efficiently SQ learnable since there is a single function); the NTK alignment \cite{jacot2018neural,cortes} that are limited to the NTK framework; the initial alignment (INAL) \cite{AbbeINAL}, which is also related to the noise-sensitivity with the advantage of depending on the network initialization at the cost of being a more implicit measure; the information exponent \cite{arous2021online,bruna1}, generative exponent \cite{bruna2} and leap \cite{abbe2023leap}, which measure when fully connected neural networks can strongly learn target functions on i.i.d.\ or isotropic input distributions and sparse or single/multi-index functions.   

We now discuss the proof techniques for Theorem \ref{3cycleTheorem}. 
We prove that the Transformer cannot learn to distinguish between $1$ cycle and $3$ cycles by means of a statistical query argument. We find a group of permutations that preserve the input distribution, take the orbit of the target function under these permutations, and show that a random pair of functions in the orbit are nearly uncorrelated. Then we argue that that means that no function is significantly correlated with a random function in the orbit, so the gradient is uninformative and the net fails to learn anything meaningful. The main complication to this argument is the fact that the input distribution is fairly complicated, in contrast to orbit arguments used in previously discussed works (e.g., for subset parities). The majority of the symbols in the input are fixed and the rest have nontrivial dependencies. However, we get around that by dividing the input into blocks and observing that switching the nonfixed symbols in any two blocks leaves the overall probability distribution unchanged. This argument would largely carry over to other cases with a sufficiently symmetric input distribution and high globality target function.

\subsection{On RASP and RASP-L}
In \cite{rasp}, the authors define a programming language, RASP, to model the types of computations a fixed-size Transformer can compute. This is more about expressivity than learnability, but to the degree that a Transformer will tend to learn a computation representable by a short RASP program if there is one that achieves low loss, it does have some implications for learning. For instance, \cite{zhou2023algorithms} observes that Transformers tend to generalize to different input lengths on tasks that can be represented by short RASP programs but not on other tasks. It also seems plausible that for tasks that can be solved by short RASP programs a sufficiently small Transformer would have a nontrivial probability of stumbling on the solution. However, we expect that randomly initialized large Transformers would tend to mix together a large number of computations in a chaotic manner, at least until they find some computation that they can improve their performance by focusing more on. So, we do not expect the existence of a short RASP program solving a problem to imply that Transformers would learn to solve it easily.

Also, RASP-L \cite{zhou2023algorithms} is a variant of RASP that does not contain the full parity function as a short program. Our globality theory predicts that the full parity can be learned by some regular Transformer (in contrast to subset parities), so this also gives a nuance. Incidentally, our reason for expecting that some regular Transformer can learn the parity is that if the Transformer is set to mostly ignore the positional embeddings then in the first layer it will essentially average all of the inputs together. The correct label is a function of this average and it only has $n+1$ possible values, so it seems likely that with the right setup it would be able to memorize the appropriate response to each of them.

\section{Additional experiments}
\subsection{Implications on random graphs}\label{sec:ER}
Here, we further discuss the disadvantages of using random graphs as the graph distribution for the implications task. There are two main downsides to using random graphs distribution instead of the cycle task distribution (Definition \ref{def:cycle-task}): 
\begin{enumerate}
    \item The distance between nodes (i.e., the number of statements to compose) does not scale well with the number of nodes/edges in the graph.
    \item Whether two nodes are connected or not often correlates with low-complexity patterns such as the degree of the nodes in random graphs, thus, weak learning on random graphs does not necessarily imply that the model has truly learned to find a path between two nodes. In other words, the model may be able to rely on shortcuts instead of solving the composition task. 
\end{enumerate}
In this section, we provide empirical evidence for both of the claims above.

First, we consider random graphs with $n$ nodes and a varying number of edges $e$. For each pair of $(n,e)$, we compute the average of maximum distance and the average of the average distance in random graphs with $n$ nodes and $e$ edges. (We ignore the nodes that are not connected.) The results for  $n=128$ are presented in Figure \ref{fig:ER-distribution}. It can be seen that the distances in the graphs do not scale well with the number of nodes and edges in the graph. (E.g., compare this to having distance $n$ in the cycle task with $2n$ nodes/edges.) 
This is because a high number of edges usually results in a very well-connected graph and a low number of edges leads to mostly isolated edges.

\begin{figure}[hbt]
    \centering
    \includegraphics[width=0.6\textwidth]{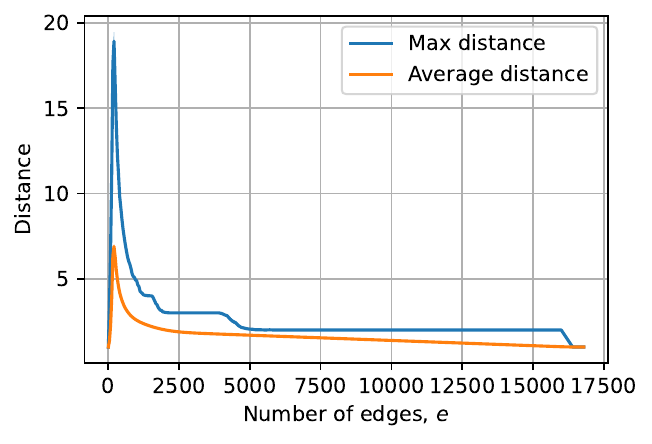}
    \caption{The average of the maximum and average distance in directed random graphs with $n=128$ nodes and a varying number of edges.}
    \label{fig:ER-distribution}
\end{figure}

Now, we move to the second claim, i.e., the model using low-complexity patterns and correlations. 
As an example, we take random graphs with 24 nodes and 24 edges. In order to have a balanced dataset with samples of mixed difficulties, we create the dataset as follows. We first sample a random graph with 24 nodes and edges. Then with probability $0.5$ we select two nodes that are not in the same connected component (label 0) and with probability $0.5$ we choose a distance $d \in \{1,2,3,4\}$ uniformly and we choose two nodes that have distance $d$ (if the graph does not have any two nodes with distance $d$, we sample another random graph). As a result, our dataset is balanced and $12.5\%$ of the samples have distance $d$ for $d \in \{1,2,3,4\}$. We trained our model on this dataset and we observed that the model reaches an average accuracy of roughly $80\%$. The results are shown in Figure \ref{fig:ER}. More precisely, we observed that the model has perfect accuracy when the two nodes are connected (there is a path), and has around $60\%$ accuracy when the two nodes are not connected (the nodes are not in one connected component). In other words, the default behavior of the model is to say that the nodes are connected and the model can also detect that two nodes are not connected in $60\%$ of the cases. 

\begin{figure}[tb]
    \centering
    \includegraphics[width=0.6\textwidth]{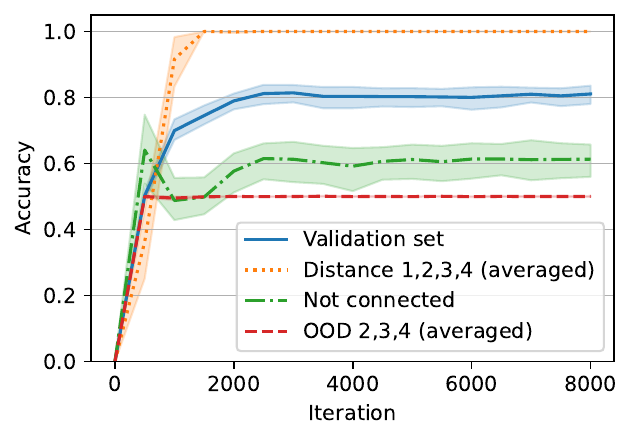}
    \caption{Performance of a model trained on a balanced distribution of random graphs with 24 nodes and edges where with probability $0.5$ the query nodes are not connected and with probability $0.5$ they are connected and their distance is uniformly selected from $1,2,3,4$. The validation set has the same distribution as the training set showing that the model reaches around $80\%$ accuracy on in-distribution samples. Particularly, the model has perfect accuracy on connected nodes (distance 1-4) and around $60\%$ accuracy on the nodes that are not connected. However, when we tested the model on OOD samples (where some spurious correlations are not present) the model showed a chance level performance. Note that these samples would be of low complexity if the model was actually checking whether there exists a path or not.}
    \label{fig:ER}
\end{figure}

To further test whether the model is truly understanding that the two nodes are not connected or it is only relying on low-complexity correlations, we designed new data distributions and assessed the model's behavior on these new distributions. The samples in the new distributions also have $24$ nodes and edges so the model does not have a length generalization problem. More specifically, for $i \in  \{2,3,4\}$, we designed distribution $OOD~i$ such that each dataset is balanced and for each sample, the two nodes are either in a cycle of size $2i$ with distance $i$ or they are in two disjoint cycles of size $i$. All the other nodes are also in different cycles.\footnote{For example, for $i=3$, distribution $OOD~3$ consists of graphs with 4 cycles of size 6 where the nodes are in a single cycle and their distance is 3 and graphs with 5 cycles of sizes 3,3,6,6,6 where the query nodes are in the two cycles of size 3.} Note that these distributions are motivated by the cycle task. For example, it is not possible for the model to merely rely on the degree of the nodes. However, if the model uses the correct algorithm (i.e., tries to find a path) then the number of reasoning steps (e.g., length of the BFS/DFS search) is $i$ as the distance between the nodes is $i$ when they are connected and otherwise they are connected exactly to $i-1$ other nodes. As it can be seen in Figure \ref{fig:ER}, the model has $50\%$ (random) accuracy on these distributions meaning that it is not really checking whether there is a path between two nodes or not, even for simple examples in $OOD~2$ supporting that the model is relying on correlations rather than finding a path. (In particular, the model always outputs connected on these OOD datasets.)

We tried to further understand the behavior of this model. By sampling, we computed that one can get an accuracy of around $82\%$ on in-distribution samples just by outputting not-connected if the out-degree of the source query node or the in-degree of the destination query node is zero and connected otherwise. Further, we noticed that this predictor has a high correlation with the output of the model. In particular, in almost all of the cases that the model predicts `not-connected', the source's out-degree or the destination's in-degree is zero. (The model may still misclassify some of such samples depending on the random seed.) The latter shows that the model is indeed relying on the degrees of the query nodes as a shortcut.

\subsection{Change of distribution and curriculum learning}\label{sec:curriculum-learning}
We have defined the cycle task such that all samples in the dataset have the same difficulty. More specifically, if the two nodes are connected their distance is $n$ and if they are not connected they are each in a cycle with $n$ vertices. Thus, it is a natural question to ask what would happen if the training distribution included samples of varying difficulties. To investigate the answer to this question, we use a distribution with samples of mixed difficulties for the training. Furthermore, we try curriculum learning \cite{curriculum} by increasing the samples' difficulty throughout the training.

\paragraph{Mixed distribution.} We change the training distribution to a uniform mixture of the cycle task distribution for sizes $i=2, \ldots, n$, i.e., each sample comes from the distribution of the cycle task for size $i$ (having $2i$ nodes and edges) with probability $\frac{1}{n-1}$. Therefore, in the mixed distribution setting, the number of reasoning steps varies between $2$ to $n$. As an example, we set $n=7$ and train our model with fresh samples from the mixed distribution. We assess the model's performance on cycle task samples of sizes $2, \ldots, n=7$ (where we include both connected and disconnected cases as in the original definition). The results are shown in Figure \ref{fig:mixed-distribution}. It can be seen that the cycle task is learned in the order of difficulty (i.e., size). Further, note that when a mixed distribution is not used, weak learning of the cycle task for $n=7$ is not possible up to $100k$ iterations (see also Figure \ref{fig:no-cot-fails}), whereas here, weak learning for $n=7$ begins in the first $30k$ iterations. Note that this is not in contrast with our theoretical results since including easy samples reduces the globality. Analogously, in the setting of learning parities, it has been shown that using a biased rather than uniform distribution (which makes the distribution simpler) can make parity targets easier to learn \cite{quantifying, cornacchia2023mathematical}.

\begin{figure}
\centering
     \begin{subfigure}[t]{0.45\textwidth}
         \centering
         \includegraphics[width=\textwidth]{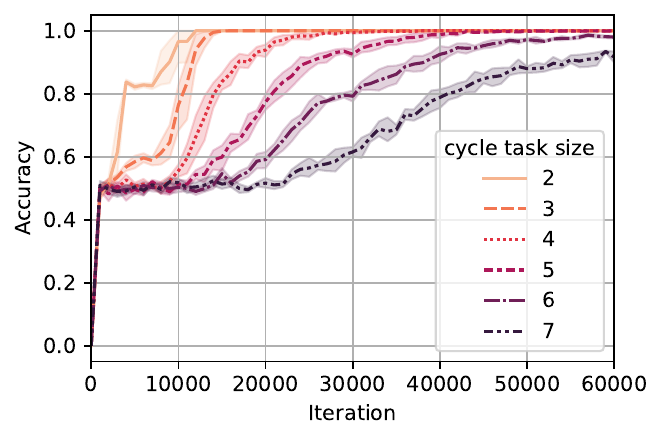}
         \caption{Mixed distribution of different sizes at training time.}
         \label{fig:mixed-distribution}
     \end{subfigure}
     \hfill
     \begin{subfigure}[t]{0.45\textwidth}
         \centering
         \includegraphics[width=\textwidth]{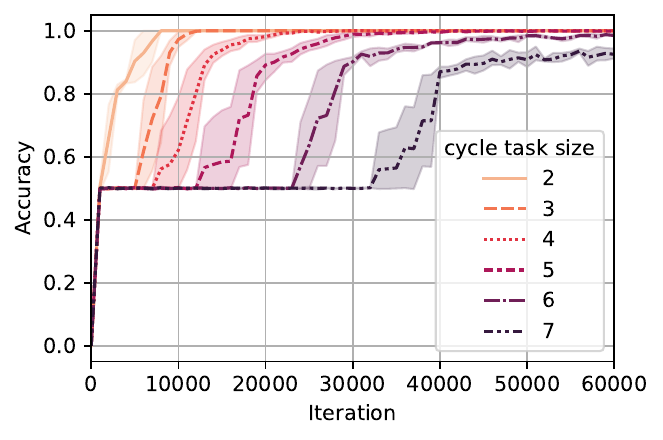}
         \caption{Curriculum learning during training where previous samples are repeated avoiding them from being forgotten.}
         \label{fig:mixed-curriculum}
     \end{subfigure}
    \caption{Accuracy for cycle tasks of varying sizes where a mixed distribution (left) and curriculum learning (right) have been used during training. It can be seen that using both a mixed distribution of samples with different difficulties and curriculum learning can reduce the learning time.}
\end{figure}

\paragraph{Curriculum learning.} Next, we try curriculum learning, i.e., we give samples in the order of difficulty (size in the cycle task) to the model during training. We consider two settings: (1) a setting in which the model has to fit samples of all difficulties and (2) a setting in which the model is allowed to forget easier samples. In other words, in the first setting, we want the model to fit cycle task samples of sizes $2, \ldots, n$ while in the second setting, we only care about fitting samples of size $n$. We start with the first setting which is closer to the notion of mixed distribution above. We consider distributions $D_2, \ldots, D_n$ such that distribution $D_i$ is a uniform mixture of cycle task samples of sizes $2, 3, \ldots, i$ (e.g., $D_n$ is the mixed distribution used for the mixed distribution setting of Figure \ref{fig:mixed-distribution}). We start training on $D_2$ and we change training distribution from $D_i$ to $D_{i+1}$ when reaching a $95\%$ accuracy on $D_i$. The results for this curriculum setting are provided in Figure \ref{fig:mixed-curriculum}. Comparing this curriculum setting to the use of mixed distribution without curriculum (Figure \ref{fig:mixed-distribution}), we can see that curriculum learning helps the model reach a high (e.g., $80\%$) accuracy slightly faster. Nevertheless, note that weak learning starts earlier in the mixed distribution setting, as the model is trained on samples of all difficulties from the beginning. The general observation that beyond using a mixed distribution, curriculum is helpful for learning has been previously shown both theoretically \cite{abbe2023provable} and empirically \cite{soviany2022curriculum}.

Now, we move to the second setting where we allow easier samples to be forgotten. More precisely, we consider distributions $D_2, \ldots, D_n$ such that distribution $D_i$ is the distribution of samples of the cycle task of size $i$ (i.e., $2i$ nodes and edges). Similarly, we start training on $D_2$ and we go from $D_i$ to $D_{i+1}$ when reaching a $95\%$ accuracy on $D_i$. We present the accuracy curves for a single random seed in Figure \ref{fig:curriculum-forget}. We further provide the average number of iterations required to reach $0.95\%$ accuracy for the cycle task of different sizes in Figure \ref{fig:curriculum-forget-iters}. It can be seen that the time complexity for this variant of the curriculum method is lower than the former curriculum method at the cost of forgetting samples of smaller sizes. 

\begin{figure}
\centering
     \begin{subfigure}[t]{0.45\textwidth}
         \centering
         \includegraphics[width=\textwidth]{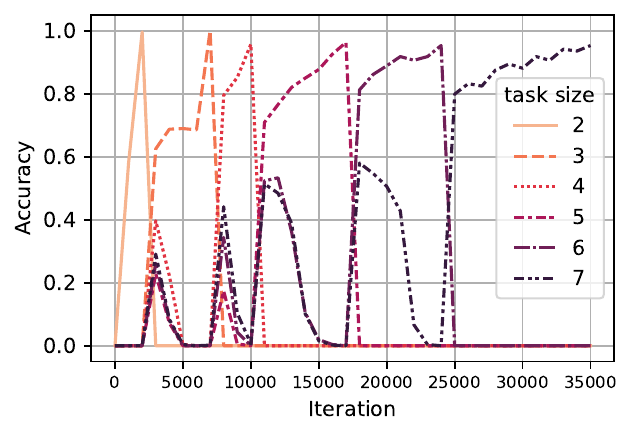}
         \caption{Accuracy curves for cycle tasks of different sizes when a curriculum is used. We can see that in this version of the curriculum where we do not repeat the samples, samples from the previous distribution are indeed forgotten by the model.}
         \label{fig:curriculum-forget}
     \end{subfigure}
     \hfill
     \begin{subfigure}[t]{0.45\textwidth}
         \centering
         \includegraphics[width=\textwidth]{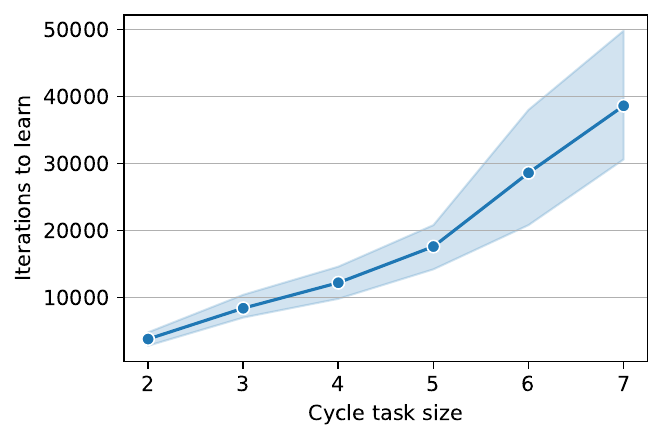}
         \caption{Average number of steps for learning the cycle task with curriculum learning where previous samples are not repeated and allowed to be forgotten.}
         \label{fig:curriculum-forget-iters}
     \end{subfigure}
    \caption{Curriculum learning based on sizes of the task used at training time. Here, samples of smaller sizes are allowed to be forgotten. The left plot presents accuracy for different sizes for a single run while the plot on the right presents the average number of iterations required for learning using curriculum for different sizes.}
\end{figure}

In sum, using distributions with samples of a mixed difficulty and also curriculum learning can reduce the learning complexity. (E.g., they made cycle task of size $7$ learnable). Nevertheless, the scratchpad approaches are still significantly more efficient (see Figure \ref{fig:cycles-scratchpad}).

\subsection{Learning parities with scratchpad}\label{sec:parities-scratchpad}
Consider $n$ bits $x_1, \ldots, x_n \in \{-1, +1\}$ with a uniform distribution over the Boolean hypercube. As discussed in Section \ref{sec:sp-parity}, learning the parity of any fixed $\frac{n}{2}$ bits is exponentially hard (e.g., an exponential number of iterations in $n$ if the model size is fixed). Note that the globality is $\frac{n}{2}$ in this setting. Here, we particularly focus on learning the parity of the first $\frac{n}{2}$ bits with a scratchpad of cumulative parities as discussed in Section~\ref{sec:sp-parity}. In other words, we design a scratchpad given by the sequence 
\begin{equation*}
    y_1=x_1,~~y_2=x_1x_2,~~y_3=x_1x_2x_3,~~\ldots,~~y_{i+1} = x_{i+1}y_i,~~ \ldots,~~y_{\nicefrac{n}{2}} = x_1x_2\cdots x_{\nicefrac{n}{2}},
\end{equation*}
where the autoregressive globality is $2$. With this scratchpad of cumulative parities, we expect the half-parity function to be learned efficiently. We confirm the latter empirically in Figure \ref{fig:parity-sp}. It can be seen that our decoder-only Transformer can efficiently learn half-parities of growing sizes (e.g., parity of the first 200 bits from 400 bits is the rightmost data point). We note that we have repeated each experiment for $10$ different random seeds and we observed that the number of iterations depends on the the random seed of the experiment. Nevertheless, for all of the seeds, the parities are learned efficiently with a low number of iterations. 

\begin{figure}[tb]
    \centering
    \includegraphics[width=0.6\textwidth]{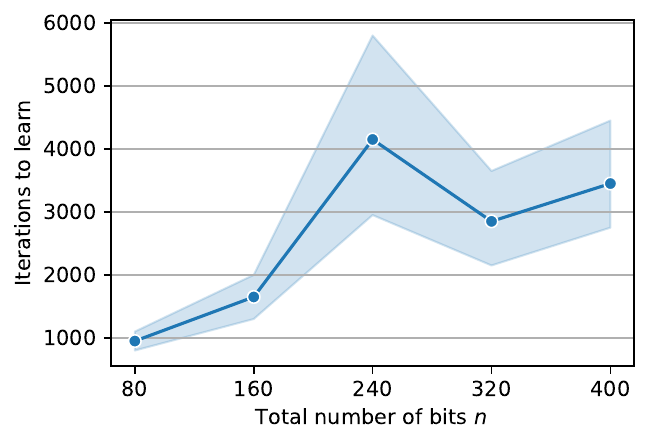}
    \caption{Learning the half-parity function (learning the parity of the first $\nicefrac{n}{2}$ bits from the total $n$ bits) for different numbers of bits using a scratchpad. It can be seen that the half-parity targets can be learned efficiently as the number of bits $n$ grows. Note that the random seed of the experiment can cause some variation in the number of iterations required for learning the parity.}
    \label{fig:parity-sp}
\end{figure}

\subsection{Length generalization for the parity task} \label{app:parity-length-gen}
Here, we provide an inductive scratchpad for the parity task which makes length generalization possible. We focus on the full parity problem, i.e., given a sequence of $0$s and $1$s determining whether the number of $1$s is even (output $0$) or odd (output $1$). 

First, we explain the input data format. We fix an ambient dimension $d_{\mathrm{amb}}$ and also assume that we have $n \leq d_{\mathrm{amb}}$ bits. Then, we choose $n$ random positions among $d_{\mathrm{amb}}$ possible ones and embed the input in dimension $d_{\mathrm{amb}}$. For the unused positions, we use a placeholder token such as a space or, here, for better presentation, an underscore. An example for $d_{\mathrm{amb}}=12$ and $n=6$ could be \texttt{\_01\_10\_0\_\_1\_}.
% random spaces idea

We design the inductive scratchpad as follows. Assuming that input has $n$ bits we define $n+1$ states for the inductive scratchpad. 
For $i \leq n$, the $i$th state is defined as 
\begin{center}
    \texttt{[<pointer>]<value of the $i$th bit>,<parity of the first $i$ bits>}
\end{center}
 where \texttt{<pointer>} determines the position of the $i$th bit and \texttt{<value of the $i$th bit>} is the value of the $i$th bit (which the model can learn to retrieve using the pointer) and \texttt{<parity of the first $i$ bits>} is the cumulative parity of the first $i$ bits that can be computed using its value in the previous state (parity of the first $i-1$ bits) and the value of the $i$th bit. For the last state, state $n+1$, we use  \texttt{[$d_\mathrm{amb}$]\_,<final parity><EOS>} where final parity is equal to the parity of the first $n$ bits that is computed in state $n$. So, one can easily check that this scratchpad is of low globality and each state can be easily computed using the previous state and question. As an example consider input \texttt{\_01\_10\_0\_\_1\_} again. The inductive scratchpad for this example can be written as 
 \begin{center}
     \texttt{<START>\textcolor{RoyalBlue}{[1]0,0\#}\textcolor{OrangeRed}{[2]1,1\#}\textcolor{Emerald}{[4]1,0\#}\textcolor{RoyalPurple}{[5]0,0\#}\textcolor{Maroon}{[7]0,0\#}\textcolor{OliveGreen}{[10]1,1\#}\textcolor{Dandelion}{[12]\_,1<EOS>}}
 \end{center}
where different states are shown by different colors. 
To show the length generalization ability of the inductive scratchpad we set $d_{\mathrm{amb}}=60$. Moreover, we train on samples with up to $n \leq N_\mathrm{train} = 30$ bits. Particularly, we use fresh samples and generate each sample such that the number of bits $n$ has a uniform distribution over $\{1, 2,\ldots, 30\}$. We train our base decoder-only model for $2000$ iterations and we test length generalization ability for different number of bits between $N_\mathrm{train} = 30$ and $d_{\mathrm{amb}}=60$. Results are reported in Figure \ref{fig:parity-length-gen}. As it can be seen the model generalizes well to inputs with $50$ bits. We also observed that the generalization ability on longer sequences (e.g., $55$ bits) is not robust and depends on the random seed of the model. 

\subsection{Length generalization for the addition task}\label{app:addition-length-gen}
Length generalization on arithmetic tasks and particularly on addition has received a surge of interest recently \cite{bubeck, jelassi2023length}. In this section, we focus on length generalization on the addition task where the model has seen the addition of numbers with up to $N_\mathrm{train}$ digits, and the model is tested on more digits $N_\mathrm{test}$. In particular, we provide two formats for the data and scratchpad which allow the model to length generalize. 

\paragraph{Inductive scratchpad using random spaces.} The first format is similar to our solution for the parity length generalization problem. First, we fix an ambient  dimension $d_{\mathrm{amb}}$. Assume, we want to add two operands with $n \leq d_{\mathrm{amb}}$ digits. We embed the two numbers and the plus sign $+$ in a $2d_{\mathrm{amb}}+1$ positions. Between them, we put spaces, or for better readability underscores. Also, we put an $=$ at the end. For example, for $x=94$, $y=31$ and $d_{\mathrm{amb}}=4$ we can have \texttt{94\_+\_3\_\_1=}. Now we explain the generation of the scratchpad. First, we generate a random sequence of tokens with size $d_{\mathrm{amb}} + 2$ such that it begins with `\$', e.g., \texttt{\$xgwg6} we call this text \texttt{ans[0]}. Now, we enter the induction mode. For $i > 0$, each state in the induction mode is given by
\begin{center}
    \texttt{s[i] = [<pointer to the $i$th digit of x>] <$i$th digit of x> [<pointer to the $i$th digit of y>] <$i$th digit of y> c <current value of the carry> r <ans[i+1]>} 
\end{center}
where $x,y$ are the operands, the pointer is counted from zero, and \texttt{ans[i+1]} is computed inductively 
\begin{center}
    \texttt{ans[i+1] = (<$i$th digit of x> + <$i$th digit of y> + <carry from the previous state> \% 10) <ans[i][:-1]>}
\end{center}
In other words, at each iteration, we shift the \texttt{ans} to the right (and lose the rightmost token). Instead, we concatenate the $i$th digit of the summation to it from the left. So in general, the model has to increase the pointers in the scratchpad, read their corresponding values, and do one summation using them (and the carry in the previous state) at each reasoning step. The scratchpad ends when both numbers are finished. Note that the answer is always the string to the left of $\$$ at the end of the text. Thus, the completed scratchpad for our example (where input is \texttt{94\_+\_3\_\_1=}) can be given by
\begin{center}
\small
\texttt{\$xgwg6<START>\textcolor{RoyalBlue}{[01]4[08]1c0r5\$xgwg\#}\textcolor{OrangeRed}{[00]9[05]3c1r25\$xgw\#}\textcolor{Emerald}{[-1]\_[03]\_c0r125\$xg<EOS>}}
\end{center}
where different states are represented in different colors.
With this scratchpad format, $N_\mathrm{train}=10$, and $d_{\mathrm{amb}}=30$ we can generalize to numbers with up to 18/20 digits (depending on the seed). See the results in Figure \ref{fig:length-gen-addition}.

\paragraph{Inductive scratchpad using shifts.} For the second solution, we also fix an ambient dimension. $d_{\mathrm{amb}}$. Assume, we want to add two operands with $n \leq d_{\mathrm{amb}}$ digits. For input formatting, we first concatenate a \$ to the left of the two operands.  Next, for each operand, we generate a random text of size $d_{\mathrm{amb}} - n$ and concatenate it to the left of the operands and call them \texttt{x[0]} and \texttt{y[0]}. Lastly, we use an $=$ to finish the question. For the scratchpad, we first generate a random text with $d_{\mathrm{amb}} +2$ tokens such that the leftmost token is \$. We call this \texttt{ans[0]}. Also, we use a variable for keeping the carry and we call it \texttt{c} which is zero at the beginning. In this format of the scratchpad, we do not use the \texttt{<START>} token meaning that the input question is also treated as a state (and is forgotten by the future states). We define each state ($i \geq 0$, $i=0$ corresponds to the question) 
\begin{center}
\texttt{s[i] = <x[i]>+<y[i]>=ans[i]|c[i]}
\end{center}
where \texttt{x[i]}, \texttt{y[i]} are computed by a right cyclic shifts of \texttt{x[i-1]}, \texttt{y[i-1]}
and \texttt{ans[i]}, \texttt{c[i]} are computed using
\begin{center}
    \texttt{ans[i] = (<RMD(x[i-1])>+<RMD(y[i-1])> + <c[i-1]> \% 10) ans[i-1][:-1]}\\
    \texttt{c[i] = 0 if (<RMD(x[i-1])>+<RMD(y[i-1])> + <c[i-1]> < 10) else 1}\\
\end{center}
where \texttt{RMD} represents the rightmost digit operator. In other words, at each step, we shift both the operands and the answer with the difference that operand shifts are cyclic, while for the answer we always add one digit of the correct answer to it from the left (and lose one digit from the right). Note that the position of the digits that we are adding always remains the same. Also, the addition is finished when \$ is the rightmost digit of both operands. When the addition is finished, we use one more state to output the final answer (possibly using the last carry variable).  For example, for $n=2$ and $d_\mathrm{amb}=4$, \texttt{fs\$46+ih\$98=} could be an input example. In this case, the scratchpad can be written as
\begin{center}
\texttt{\$kckn|0\#\textcolor{RoyalBlue}{6fs\$4+8ih\$9=4\$kck|1\#}\textcolor{OrangeRed}{46fs\$+98ih\$=44\$kc|1\#}\textcolor{Emerald}{144\$kc<EOS>}}
\end{center}
where different colors are used for different states. 
In the example above one can note that some part of the \texttt{s[0]} is in the question and some part of it is in the scratchpad.

Note that the formatting required for the input of this scratchpad is stronger than the previous scratchpad. However, here we can get a much stronger length generalization. By setting $d_\mathrm{amb}=30$ and training on numbers with up to $4$ digits, we can generalize to numbers with up to $26$ digits and even $30$ digits for some of the seeds (see Figure \ref{fig:length-gen-addition}).

For both addition methods, we sampled numbers with $n\leq N_\mathrm{train}$ digits with a probability proportional to $n$. Nevertheless, we did not observe much dependency on the distribution. 
Also, note that in the design of scratchpads for both addition and parity, we have used techniques such as inserting random spaces and embedding in a fixed dimension. We note that this is unavoidable for the input as we have used absolute positional embedding for it. Nevertheless, we expect that by using an inductive scratchpad along with relative positional embedding and pre-training, one would be able to achieve the same length generalization with fewer assumptions on the format of the input.

\section{Experiment and implementation details}\label{app:experiment-details}
\subsection{Architecture and datasets}
In all of the experiments, we use GPT2-style \cite{gpt2} decoder-only Transformers and we train them from scratch in an autoregressive manner with the cross-entropy loss and AdamW \cite{adamw} optimizer. Our implementation uses the PyTorch framework \cite{torch} and is mostly built on NanoGPT's implementation \cite{nanoGPT2023}. In particular, our Transformers use causal attention masking and absolute learnable positional embeddings. For most experiments, we use a small model with $6$ layers, $6$ heads, and an embedding dimension of $384$ which results in a model with approximately $10M$ parameters.\footnote{The exact number depends on the task and the vocabulary size of it.} We only change the size of the model in Figure \ref{fig:no-cot-fails} where we use models with 
 $8$ layers, $8$ heads, and an embedding dimension of $512$ (approximately $25M$ parameters), and $12$ layers, $12$ heads, and an embedding dimension of $768$ (roughly $85M$ parameters).

For the cycle task, we use $1000$ node names formatted like \texttt{v123}. For simplicity of analysis, we regard each node name as a single token. Other than that and for the other tasks, we treat each character as a single token.  

For the length generalization experiments, we realized that the performance of the model depends on the random seed of the network. So we repeated each experiment 10 times and reported the median of the results in the plots (along with other runs). For other experiments, we did not observe much variation between seeds and repeated each experiment 5/10 times and reported $95\%$ confidence intervals. We used different Nvidia GPU devices for running our experiments including H100, A100, and RTX4090. We approximate that the runtime for experiments presented in this paper is around 200 hours (excluding hyperparameter search). 

Our code is publicly available at \url{https://github.com/aryol/inductive-scratchpad}.

\subsubsection{Hyperparameter tuning and sensitivity}
In general, we tried different values for the learning rate (with different schedules), batch size, dropout, and weight decay. For different tasks, we have used the hyperparameters that were the most stable and fast. 
The most significant sensitivity is to the batch size. We often find that larger batch sizes help with the training to the point that some tasks cannot be learned with batch sizes smaller than $256$.\footnote{For some experiments, we increased the gradient accumulation steps instead of the batch size to get the same effect with less memory consumption.} 
Also, in the length generalization experiments, we observed that the experiments are rather sensitive to the dropout and weight decay parameters. Generally, strong regularizations can increase the uncertainty of the models. Considering the large number of tokens in the scratchpad of the length generalization experiments and their sensitivity, this uncertainty can increase the probability of error. Of course, a weak regularization can also result in a worse generalization.  In addition, for most of the experiments, we used either fresh samples or a rather large number of samples as our objective is mostly measuring the time complexity or OOD generalization. The exact value of hyperparameters for each task is available in our code. 

\subsection{Implementation of the inductive scratchpad}\label{app:inductive-sp-implement}
Here we describe the implementation of the inductive scratchpad.  Assume that the question and scratchpad sequence are given by \texttt{Q<START>s[1]\#s[2]\#\dots \#s[k]<EOS>}. Note that \texttt{Q} and \texttt{s[i]}s can each contain a number of tokens. Moreover, note that \texttt{Q} can also include a part of the scratchpad besides the question. Anything that comes before \texttt{<START>} acts like a permanent memory and the model can attend to it for the generation of all states. So for example, if there is some shared information between all states, it is more efficient to put it in the scratchpad before \texttt{<START>} rather than including it in all of the states. 
Note that, in general, our goal is to only learn the induction function. In other words, we want to generate tokens of the $i$th state \texttt{s[i]} as if the sequence to this point was only \texttt{Q<START>s[i-1]\#}. We now explain how this can be achieved during training and generation. 

First, we provide two solutions for training time. One simple  way is to break the original sequence \texttt{Q<START>s[1]\#s[2]\#\dots \#s[k]<EOS>} into 
\begin{center}
    \texttt{Q<START>s[1]}, \texttt{Q<START>\#s[1]\#s[2]\#}, \dots, \texttt{Q<START>s[k-1]\#s[k]<EOS>}.
\end{center}
 We also need to make sure that no loss is not computed for \texttt{Q<START>s[i]\#} part of \texttt{Q<START>s[i]\#s[i+1]\#} for $1 \leq i < k$ which can be easily done using a loss mask. This approach ensures that the loss is computed once for the question and each state and also each state is generated from the previous state and the question in an identical manner. Note that all of these operations are done once as a preprocessing step. Now, we describe a second implementation method for the train time. We first duplicate each state other than the last state, i.e., \texttt{Q<START>s[1]\#s[1]\#s[2]\#s[2]\#\dots \#s[k]<EOS>}. Next, we group the consecutive states, reindex the position of the tokens, and design attention masks and loss masks as follows:
\begin{center}
{\footnotesize
\setlength\tabcolsep{4pt}
\begin{tabular}{c|c|c|cc|cc|c|cc}
    Tokens & \texttt{Q<START>} &  \texttt{s[1]\#} &  \texttt{s[1]\#} &  \texttt{s[2]\#}  &  \texttt{s[2]\#} &  \texttt{s[3]\#} & \dots & \texttt{s[k-1]\#} & \texttt{s[k]<EOS>}  \\\hline\hline
    Loss & \multirow{2}{*}{$1$} & \multirow{2}{*}{$1$} & \multirow{2}{*}{$0$} & \multirow{2}{*}{$1$} & \multirow{2}{*}{$0$} & \multirow{2}{*}{$1$} & \multirow{2}{*}{$\cdots$}  & \multirow{2}{*}{$0$} & \multirow{2}{*}{$1$}\\
    mask & & & & &&& \\\hline
    Attention & \multirow{2}{*}{$0$} & \multirow{2}{*}{$1$} & \multicolumn{2}{c|}{\multirow{2}{*}{$2$}} & \multicolumn{2}{c|}{\multirow{2}{*}{$3$}} & \multirow{2}{*}{$\cdots$} & \multicolumn{2}{c}{\multirow{2}{*}{$k$}} \\
    group &&&&&&& \\\hline
    Positional & \multirow{2}{*}{$0,1, \ldots, t-1$} & \multirow{2}{*}{$t, t+1, \ldots$} & \multicolumn{2}{c|}{\multirow{2}{*}{$t, t+1, \ldots$}} & \multicolumn{2}{c|}{\multirow{2}{*}{$t, t+1, \ldots$}} & \multirow{2}{*}{$\cdots$} & \multicolumn{2}{c}{\multirow{2}{*}{$t, t+1, \ldots$}} \\
     indices &&&&&&&& \\
\end{tabular}
}
\end{center}

where we have assumed that $t$ tokens have appeared before the first state \texttt{s[1]}. Note that this number may vary across samples. We can easily create an attention mask based on the attention groups. All groups only attend to tokens in their group and tokens in group $0$ (everything that comes before \texttt{<START>}). Also, using the loss mask vector, we do not compute the loss for the second copy of each state (where the loss mask is equal to zero). Also, for each state generation, we reset the positional indices of the tokens. Using the appropriate loss mask, attention mask, and positional indices explained above guarantees that the loss is computed exactly once for each state, and each state is generated based on the previous state and tokens before \texttt{<START>} (i.e., the question \texttt{Q}) in an identical manner. Note that both approaches that we discussed for the train time can be implemented easily for conventional Transformer models. Further, they do not change Transformer models' behavior on non-inductive samples. Our methods also work with both absolute and relative positional embeddings. Note that the first approach favors Transformers with a small block size (context length) and the second approach is more suitable when the block size is large. So based on the block size, a mixture of these two approaches can be used. In our implementation, we use the second approach as we do not have an issue with the block size. 

Now, we discuss token generation. Assume that we are generating the tokens of $i$th state \texttt{s[i]}, i.e., \texttt{Q<START>s[1]\#s[2]\#\dots s[i-1]\#} have already been generated. To have the inductive behavior, we need the model to generate tokens of \texttt{s[i]} using only the last state and the tokens before \texttt{<START>}. Similar to the training time, this is achievable through using two methods. The first method is to change the input of the model, basically, we can give \texttt{Q<START>s[i-1]\#} as the input to the model when generating tokens for \texttt{s[i]}. Alternatively, we can keep the input intact and just use an attention mask that prevents \texttt{s[i]} tokens from attending to any token other than tokens of  \texttt{Q<START>} and \texttt{s[i-1]\#}. Similar to the training time, one also needs to reindex the position of tokens of \texttt{s[i-1]\#} and \texttt{s[i]\#} so that they appear exactly after \texttt{Q<START>}. Note that it is still possible to do key-value (KV) caching \cite{kvcache-2019-fairseq, 2023keyformer-kvcache} to increase the decoding speed. KV caching stores previous keys and values corresponding to the previous tokens and does not compute them again at the expense of the memory. Generally, for KV caching, we are only in trouble when going from $(i-1)$th state to the $i$th state, because the current keys and values of tokens of \texttt{s[i-1]} are computed based on \texttt{s[i-2]}. However, for the generation of \texttt{s[i]}, we only attend to \texttt{s[i-1]} and do not want \texttt{s[i-1]} to attend to any of the previous states to conserve the inductive behavior. One solution to this problem is to compute the key-value pairs for tokens of \texttt{s[i-1]} again with the correct positional indices and attention masking once we have the transition from \texttt{s[i-1]} to \texttt{s[i]}. Alternatively, one can always cache two versions of keys and values, one for the generation of the current state and one for the generation of the future state. 

We note that in the inductive scratchpad, all the states except the penultimate one are ignored. As a result, the effective number of tokens is significantly reduced compared to the full scratchpad. Consequently, the inductive scratchpad works better for longer scratchpads and models with more limited context sizes and generally scales better. 

Also, note that the inductive scratchpad can generally be used for a wide range of algorithmic tasks. For example, consider an algorithm with a for loop that updates some variables. One can easily put the variables of the algorithm in the state of an inductive scratchpad, and use the Transformer for computing the values at each iteration (and also determining the halting of the loop). The inductive scratchpad for the parity, addition, and cycle tasks all fall into this category.
\subsection{Comparison with relative positional embeddings}\label{app:indutive-relative-PE}
In this section, we discuss whether it is possible to induce an inductive structure for the scratchpad using relative positional embedding \cite{Shaw2018SelfAttentionWR, Dai2019TransformerXLAL} instead of using the explicit inductive scratchpad format introduced in this paper. 
For an inductive structure to work, we want each state to be generated using only the tokens of the previous state and the question in an identical manner for different states.  

More precisely, assume that the question and scratchpad are given by \texttt{Q<START>s[1]\#s[2]\#\dots \#s[k]<EOS>} (one can also decide to remove \texttt{<START>} and \texttt{\#} tokens). For simplicity, assume that the size of all states (i.e., the number of tokens in each state) is equal to $T$ (where we also include \texttt{\#} if it is used). Relative positional embeddings compute the attention between two tokens based on their distance instead of their absolute positions in the sequence. Therefore, the attention pattern between  \texttt{s[i+1]} and \texttt{s[i]} is similar to the attention pattern between \texttt{s[i]} and \texttt{s[i-1]}, and one could hope for an inductive structure to emerge. 

There are a few obstacles, however:
\begin{enumerate}
    \item The distance between the states and question tokens increases as each state is generated. However, in order to have an inductive structure, we want the attention pattern between different states and the question not to change. One could think of using encoder-decoder architectures where the question is in the encoder part and computing the cross-attention between states and the question ignoring the positions of the state tokens in the decoder part. However, even this approach would lose some information. I.e., the attention between the scratchpad tokens and the question tokens cannot use the position of the scratchpad tokens within a state.
    \item Most of the relative positional embeddings \cite{raffel2019exploring,alibi,rotary} allow tokens to attend to all other tokens. Even if one uses an attention mechanism that limits the attention to the $D$ previous tokens, after $L$ transformer layers, tokens of up to distance $DL$ can attend to each other. So in general, it is most likely, that tokens of \texttt{s[i]} can attend to tokens of other states as well as tokens of \texttt{s[i-1]} which hinders the inductive behavior.
    \item Similarly, assume that the $T$th (last) token of \texttt{s[i]} attends to all tokens of \texttt{s[i-1]} including its first token. Note that the distance between the last token of state \texttt{s[i]} and the first token of state \texttt{s[i-1]} is $2T-1$. As a result, the first token of \texttt{s[i]} would also attend to all tokens of \texttt{s[i-2]} except the first one in a similar manner (because the distance between the first token of \texttt{s[i]} and the second token of \texttt{s[i-2]} is $2T-1$) which is undesirable. 
\end{enumerate}

Note that in the analysis above we assumed a fixed number of tokens per state. If the number of tokens in each state varies, issues like (2) and (3) above can become more difficult to handle for a model that only relies on relative positional embedding. To sum up, there is a minor similarity between relative positional embeddings and the idea of induction in the scratchpad. Factors such as pre-training and the recency bias of attention may promote inductive behavior in the model to some extent. Nevertheless, there is no reason to think that the model can implement the inductive behavior relying only on relative positional embeddings for a general inductive task. In general, one can use the inductive scratchpad idea along with relative positional embedding to get better length generalization. Note that inductive scratchpad can generalize on the number of training steps. However, understanding an input with growing length still requires relying on relative positional embeddings. 

As an alternative solution, independent of positional embeddings being absolute or relative, one can use special tokens \texttt{<START>} \texttt{\#} without hard-coding their meaning and hope the model realizes the meaning of these two tokens on its own and implement a soft inductive structure (same attention masking but in a soft manner). However, such an event is also very unlikely for moderate amounts of inductive data. 

\section{Further discussion and specification of Conjecture \ref{gen_conj}}\label{app:conj-specification}

\begin{definition}[Well-behaved distribution for Conjecture \ref{gen_conj}]\label{def:well-behaved}
Input distribution $P_X$ over alphabet $\mathcal{A} ^n$ with $|\mathcal{A}|=O(n^c)$ is well-behaved if for $X\sim P_X$ there is no value that $X$ takes with probability $\Omega(n^{-c_1})$ for any $c_1>0$, and every eigenvalue of $X$'s covariance matrix of the indicator functions for the values of $X$'s entries has absolute value $O(n^{-c_2}\sum_{i=1}^n \mathrm{Var}(X_i))$ for some $c_2>0$. 
\end{definition}
In other words, the first requirement means that there is no value of $X$ that is frequent enough that allows the model to weakly learn the function simply by memorizing the value of that input. The second rules out the possibility of an input distribution with a few important components that are easy to notice and determine the output, such as a scenario where there are $\log(n)$ blocks of $n/(2\log(n))$ bits such that all of the bits in a block are always the same and the output depends only on the values of the blocks. Other than that, many distributions of interest are well-behaved, such as the i.i.d. measure as in \cite{abbe2020poly,arous2021online,AbbeINAL,abbe2023leap,bruna1,bruna2}, or measures with dependencies as for the cycle task considered in this paper.

\textbf{Intuition on Conjecture \ref{gen_conj}.} To see the role of the histogram, $\hat P_X$, note that if one removes positional embeddings from a Transformer, the Transformer would become permutation invariant. In this case, learning functions like full parity (parity of all the bits) becomes very easy as there are only $n+1$ possibilities for $n$ bits in the eyes of a permutation invariant model. One could potentially achieve a similar effect to removing the positional embeddings, if one initializes positional embeddings with a small enough (or vanishing) scale. That is the reason that $\hat P_X$ needs to be included in the definition of the globality degree. Further see Appendix \ref{app:proof-thm1} for more intuition on Conjecture \ref{gen_conj}.

\section{Intuition on the agnostic scratchpad}\label{app:intuition-agnostic-scratchpad}
In a lot of the cases where there is a scratchpad, we assume that during training we know what the Transformer should write in the scratchpad and just need to teach it to do so. However, this tends to be an unrealistic assumption in practice. If we knew how to compute the appropriate scratchpad entries from the inputs then we could just give the Transformer a prefilled scratchpad instead of training it to write the appropriate entries in the scratchpad. Admittedly, we could have a scenario where it is expensive for us to compute the appropriate scratchpad entries and we are hoping to teach the Transformer to do it for us, but there is a limited middle ground between it being cheap enough that we do not need the Transformer to compute the entries and it being too expensive to compute the correct scratchpad entries for all the training data. The other case where we would know what the correct scratchpad entries were during training is if the training data already came with scratchpad entries, which could happen if our training data includes some kind of explanation, which might not often be the case.

So, we could easily have a case where we are trying to train a Transformer with a scratchpad but do not know what it should write in the scratchpad. In that case, it seems like the best available way to evaluate whether an entry the model wrote in the scratchpad is right or not is to judge it based on the end result of writing that entry. In other words, we define a loss function for entries in the scratchpad by taking the scratchpad with that entry, extending it to a full scratchpad, and checking if the transformer ends up giving the right answer. This is the mechanism behind the algorithm proposed in Conjecture \ref{conjeture-agnostic} for optimizing the network based on agnostic scratchpads.

\section{Proof of Theorem \ref{3cycleTheorem}}\label{app:proof-thm1}

In order to prove that a T-regular Transformer cannot learn to distinguish between the case where there are 3 cycles and the case where there is only one, we will take advantage of the fact that the probability distribution of the positional embeddings is invariant under permutations of the embeddings. So, if such a Transformer could learn to solve this problem it would also be able to learn a version where an arbitrary permutation was applied to the inputs. However, we will show that no function has a significant correlation with a random element of the orbit of this function under reordering of its inputs. That in turn will imply that a Transformer trained by gradient descent fails to learn anything meaningful when trained on this function. However, before we do that we will need to consider the setup for the theorem more carefully. The input in this problem is a series of blocks, each of which specifies how $a_{i-1}$, $b_{i-1}$, and $c_{i-1}$ connect to $a_i$, $b_i$ and $c_i$. So, each of these blocks can be viewed as representing a permutation in $S_3$, in which case their product will give the needed information on the overall graph structure. We can effectively reorder these permutations by permuting the tokens in the input. So, in order to show that permuting the tokens can completely alter the function it suffices to show that the function taking the product of a series of permutations is largely uncorrelated to the function taking a product of the permutations in a different order, which can be made rigorous as follows.

\begin{lemma}
Let $P_X$ be the probability distribution on the subset of $S_3^n$ with an even number of odd permutations and $f:S_3^n\rightarrow\{-2,1\}$ be the function such that $f(X)$ is $-2$ if $\prod_{i=1}^n X_i$ is the identity and $1$ otherwise. Next, select $z\in\{0,1\}^{\lfloor (n-1)/2\rfloor}$ uniformly at random. Then, let $\pi=\prod_{i\le \lfloor (n-1)/2\rfloor:z_i=1} (2i-1,2i)$. In other words, $\pi$ is the permutation that switches its $(2i-1)$th and $(2i)$th inputs if $z_i$ is $1$ and leaves them otherwise. Now, let $f'(x)=f(\pi(x))$ for all $x$. Then
\[\mathbb{E}_{z}[\mathbb{E}_{X\sim P_X}^2[f(X)\cdot f'(X) ]]\le 8\cdot 2^{-n/3}\]
\end{lemma}

\begin{proof}
We start by considering a single pair of permutations $\sigma,\sigma'$ drawn uniformly and independently from $S_3$ and comparing $\sigma\sigma'$ to $\sigma'\sigma$. At this point, it is helpful to think of $\sigma$ and $\sigma'$ as function on $\mathbb{F}_3$ of the form $ax+b$ with $a\in\{1,-1\}$ and $b\in\{0,1,2\}$. If $\sigma(x)=x+b$ and $\sigma'(x)=x+b'$ then clearly $\sigma\circ\sigma'=\sigma'\circ\sigma$. However, if $\sigma(x)=x+b$ and $\sigma'(x)=-x+b'$ then $\sigma(\sigma'(x))=-x+b'+b$ while $\sigma'(\sigma(x))=-x-b+b'$. The cases where $\sigma(x)=-x+b$ are similar. So, conditioned on any fixed values of the signs of $\sigma$ and $\sigma'$ in which at least one of them is odd, $\sigma\sigma'\sigma^{-1}\sigma^{\prime-1}$ is equally likely to be any even permutation. That means that if this holds and $\sigma''$ is another random permutation of a known sign then the probability distribution of $(\sigma\sigma'\sigma'',\sigma'\sigma\sigma'')$ is the uniform distribution on pairs of permutations of the correct sign. That in turn means that conditioned on the signs of all of the elements of $X$ and the value of $z$, the values of $f(X)$ and $f'(X)$ are independent of each other unless $X_{2i-1}$ and $X_{2i}$ are both even permutations for every $i$ for which $z_i=1$. So, for any fixed value of $z$, $\mathbb{E}_{X\sim P_X}[f(X)\cdot f'(X)]=2\cdot 4^{-|\{i:z_i=1\}|}$. That in turn means that \[\mathbb{E}_{z}[\mathbb{E}_{X\sim P_X}^2[f(X)\cdot f'(X) ]]=4\cdot (17/32)^{\lfloor (n-1)/2\rfloor}\le 8\cdot 2^{-n/3}\]
\end{proof}

\begin{corollary}
Let $P_X$ be the probability distribution on the subset of $S_3^n$ with an even number of odd permutations and $f:S_3^n\rightarrow\{-2,1\}$ be the function such that $f(X)$ is $-2$ if $\prod_{i=1}^n X_i$ is the identity and $1$ otherwise. Next, select $z\in\{0,1\}^{\lfloor (n-1)/2\rfloor}$ uniformly at random. Then, let $\pi=\prod_{i\le \lfloor (n-1)/2\rfloor:z_i=1} (2i-1,2i)$. In other words, $\pi$ is the permutation that switches its $(2i-1)$th and $(2i)$th inputs if $z_i$ is $1$ and leaves them otherwise. Now, let $f'(x)=f(\pi(x))$ for all $x$, and let $g:S_3^n\rightarrow [-1,1]$ be a function. Then
\[\mathbb{E}_{z}[\mathbb{E}_{X\sim P_X}^2[f'(X)\cdot g(X)]]\le 3\cdot 2^{-n/6}\]
\end{corollary}

\begin{proof}
First of all, let $(z',\pi',f'')$ be drawn from the same probability distribution as $(z,\pi,f')$ but independently of it. We have
\begin{align*}
\mathbb{E}_{z,z'}[\mathbb{E}_{X\sim P_X}^2[f''(X)\cdot f'(X)]]&=\mathbb{E}_{z,z'}[\mathbb{E}_{X\sim P_X}^2[f(\pi'(X))\cdot f(\pi(X))]]\\
&=\mathbb{E}_{z,z'}[\mathbb{E}_{X\sim P_X}^2[f(x)\cdot f(\pi(\pi'(X)))]]\\
&\le 8\cdot 2^{-n/3}.
\end{align*}
Using the inequality that given a fixed vector $v$ and probability distribution over vectors $P_u$, $\mathbb{E}_{u\sim P_u}[(v\cdot u)^2]\le ||v||_2^2\sqrt{\mathbb{E}_{u,u'\sim P_u}[(u\cdot u')^2]}$, which follows from 
\begin{align*}
&\mathbb{E}_{u\sim P_u}[(v\cdot u)^2]\\
&=\mathbb{E}_{u\sim P_u}[(v\otimes v) \cdot (u\otimes u)]\\
&=(v\otimes v) \cdot \mathbb{E}_{u\sim P_u}[(u\otimes u)]\\
&\le ||v\otimes v||_2 ||\mathbb{E}_{u\sim P_u}[(u\otimes u)]||_2\\
&= ||v||_2^2\sqrt{\mathbb{E}_{u,u'\sim P_u}[(u\cdot u')^2}]
\end{align*}
where $u$ and $u'$ are drawn independently from $P_u$ in the last expectation, 
we have 
\begin{align*}
&\mathbb{E}^2_{z}[\mathbb{E}_{X\sim P_X}^2[f'(X)\cdot g(X)]]\\
&\le \mathbb{E}^2_{X\sim P_X}\left[g^2(X)\right]\mathbb{E}_{z,z'}\left[\mathbb{E}^2_{X\sim P_X}\left[f'(X)\cdot f''(X)\right]\right]\\
&\le 8\cdot 2^{-n/3}.
\end{align*}
\end{proof}

At this point, we are finally ready to prove Theorem \ref{3cycleTheorem} as follows.
\begin{proof}
First, observe that in the setup in Theorem \ref{3cycleTheorem}, an input consists of a series of blocks where the $i$th block specifies where the edges from $a_{i-1}$, $b_{i-1}$, and $c_{i-1}$ go. These edges always go to $a_i$, $b_i$, and $c_i$ in some order (Regarding $a_n$, $b_n$, and $c_n$ as alternate names for $a_0$, $b_0$, and $c_0$). So, each block can be viewed as a permutation in $S_3$. Call these permutations $\sigma_1,...,\sigma_n$. The first $n-1$ of these permutations are mutually independent, but the requirement that there be $1$ or $3$ cycles rather than $2$ forces the last one to take on a value for which their product is even. If there are $3$ cycles their product is the identity and if there is one cycle their product is one of the other even permutations with which it is determined by whether the vertex $n$ edges from $a_0$ is $b_0$ or $c_0$. So, the probability distribution of $\sigma$ is the uniform distribution on the subset of $S_3^n$ with an even number of odd permutations and the label specifies whether or not $\prod_{i=1}^n\sigma_i$ is the identity.

The positional encodings are iid and there is no causal masking, so the probability distribution of the Transformer is symmetric under permutations of the inputs. In particular, for an arbitrary permutation $\pi\in S_n$ we can move all of the letters in block $i$ to the corresponding position in block $\pi(i)$ for each $i$ to essentially apply this permutation to the $\sigma$. Now, let $Y$ be the indicator function for $\prod_{i=1}^n\sigma_i$ being the identity. Then, for every select $z\in\{0,1\}^{\lfloor (n-1)/2\rfloor}$ let $\pi_z=\prod_{i\le \lfloor (n-1)/2\rfloor:z_i=1} (2i-1,2i)$ and $Y_z$ be the indicator function for $\prod_{i=1}^n\sigma_{\pi_z(i)}$ being the identity. By the symmetries of the Transformer, it will have as hard a time learning to compute $Y$ as it would have learned to compute $Y_z$ for any $z$. Next, let $Y_\emptyset$ be $1$ with probability $1/3$ and $0$ with probability $2/3$, independently of $X$. We claim that the probability distribution of the Transformer's weights when it is trained on $Y_z$ will be essentially the same as its weights when it is trained on $Y_\emptyset$. In order to formalize that, let $T_{w}(x)$ be the output the Transformer gives when its edge weights and positional encoding take on the values given by $w$ on an input of $x$. Next, let $P_{W,\emptyset,t}$ be the probability distribution of the weights and positional embeddings after $t$ time steps of training on $Y_\emptyset$ and $P_{W,z,t}$ be the probability distribution of the weights and positional embeddings after $t$ time steps of training on $Y_z$. For each $t$, $P_{W,\emptyset,t}$ is symmetric under permutations of the positional embeddings.

By the assumption on hyperparameters, we are using gradient descent with a clipped gradient with some $B$ polynomial in $n$ such that if there exist $(X,Y)$ for which any entry of the gradient of the loss with respect to the edge weights has absolute value higher than $B$ it is reduced to $\pm B$ when calculating the overall gradient. Let $()_B$ denote this clipping operator. For any given set of edge weights and positional embeddings $w$ and a random $z$, the expected square of the difference between the $i$th elements of the clipped gradient of the loss when the net is trained on $Y_z$ and the clipped gradient of the loss when the net is trained on $Y_\emptyset$ is
\begin{align*}
&\mathbb{E}_z\left[\left(\mathbb{E}_{X}\left[\left(\frac{d L(Y_z,T_w(X))}{d w_i}\right)_B\right]-\mathbb{E}_{X,Y_\emptyset}\left[\left(\frac{d L(Y_\emptyset,T_w(X))}{d w_i}\right)_B\right]\right)^2\right]\\
&=\mathbb{E}_z\left[\mathbb{E}^2_{X}\left[(Y_z-1/3)\cdot\left(\left(\frac{d L(1,T_w(X))}{d w_i}\right)_B-\left(\frac{d L(0,T_w(X))}{d w_i}\right)_B\right) \right]\right]\\
&\le (1/3)^2(2B)^2\cdot 3\cdot 2^{-n/6}=(4B^2/3)\cdot 2^{-n/6} 
\end{align*}
by the previous corollary. There are a polynomial number of weights and a small difference in gradients results in a total variation distance between the probability distributions of the weights one step later than is at most polynomially larger, so for any $t$, we have that
\[\mathbb{E}_z [TV(P_{W,z,t+1},P_{W,\emptyset,t+1})]\le \mathbb{E}_z [TV(P_{W,z,t},P_{W,\emptyset,t})]+poly(n)\cdot 2^{-n/12} \]
That in turn means that
\[TV(P_{W,0,t},P_{W,\emptyset,t})=\mathbb{E}_z [TV(P_{W,z,t},P_{W,\emptyset,t})]=poly(n)\cdot 2^{-n/12}\]
for all $t$ polynomial in $n$. Now, let $y_0$ be chosen to minimize the value of $(1/3)L(1,y_0)+(2/3)L(0,y_0)$. Then, let $L'(p,q)=\min(L(p,q),\max(L(1,y_0),L(0,y_0)))$ for all $p$ and $q$ be a bounded version of the loss function. The previous argument shows that  the expected bounded loss of the Transformer on $(X,Y)$ after $t$ training steps is within $poly(n)\cdot 2^{-n/12}$ of the expected bounded loss of a Transformer trained on $(X,Y_\emptyset)$ for $t$ steps and then tested on $(X,Y)$ which is
\begin{align*}
&\mathbb{E}_{w\sim P_{W,\emptyset,t}}[\mathbb{E}_X[L'(Y,T_w(X)) ]]\\
&=\mathbb{E}_{w\sim P_{W,\emptyset,t}}[\mathbb{E}_z[\mathbb{E}_X[L'(Y,T_w(X)) ]]]\\
&=\mathbb{E}_{w\sim P_{W,\emptyset,t}}[\mathbb{E}_{X,Y_\emptyset}[L'(Y_\emptyset,T_w(X)) ]]+\mathbb{E}_{w\sim P_{W,\emptyset,t}}[\mathbb{E}_z[\mathbb{E}_{X,Y_\emptyset}[L'(Y,T_w(X))-L'(Y_\emptyset,T_w(X)) ]]]\\
&= \mathbb{E}_{w\sim P_{W,\emptyset,t}}[\mathbb{E}_{X,Y_\emptyset}[L'(Y_\emptyset,T_w(X)) ]]\\
&\phantom{xxxxxx}+\mathbb{E}_{w\sim P_{W,\emptyset,t}}[\mathbb{E}_z[\mathbb{E}_X[(Y-1/3)\cdot(L'(1,T_w(X))-L'(0,T_w(X))) ]]]\\
&\ge \mathbb{E}_{w\sim P_{W,\emptyset,t}}[\mathbb{E}_{X,Y_\emptyset}[L'(Y_\emptyset,y_0) ]]\\
&\phantom{xxxxxx}-\mathbb{E}_{w\sim P_{W,\emptyset,t}}\left[\sqrt{\mathbb{E}_z[\mathbb{E}^2_X[(Y-1/3)\cdot(L'(1,T_w(X))-L'(0,T_w(X))) ]]}\right]\\
&=\mathbb{E}_{X,Y_\emptyset}[L(Y_\emptyset,y_0) ]-O(2^{-n/12})
\end{align*}
as desired.
\end{proof}
\begin{remark}
If we used batch gradient descent with a polynomial batch size instead of full gradient descent that would essentially add an inverse polynomial perturbation to the gradient. In the version of this where we are storing weights and embeddings to a limited degree of precision then for a large enough polynomial batch size the expected number of times this perturbation actually causes a parameter to get set to a different value than it would otherwise have is $O(1/n)$. In this case, we could show that with probability $1-n^{-\omega(1)}$ there are at most $\log(n)$ cases of a weight being set differently than it otherwise would have as a result of the limited batch size. There are only $2^{O(\log^2(n))}$ choices of which weights get changed in which steps and how they are changed. For any one of these options the expected loss is still within $poly(n)\cdot 2^{-n/12}$ of what it would be if the net was run entirely on random labels. So, even assuming we end up with the best option the net will still have a loss that is at best $n^{-\omega(1)}$ better than that attained by ignoring the input and always returning $y_0$. If instead of having limited parameter precision we instead have inverse polynomial noise we can still make a variant of this argument but we would need to be more careful about exactly how the net differs when the perturbation to the gradient affects the results.
\end{remark}

\subsection{Extension to agnostic scratchpads}

Theorem \ref{3cycleTheorem} can also be generalized to Transformers trained with agnostic scratchpads in order to get the following.

\begin{theorem}\label{agnostic3cycleTheorem}
Let $G$ be a directed graph which consists of a cycle of length $3n$ with probability $2/3$ and 3 cycles of length $n$ otherwise. Next, if there are $3$ cycles pick one vertex from each and if there is one cycle pick three vertices that are each $n$ edges apart. Then, label uniformly at random these vertices with $a\_0$, $b\_0$, $c\_0$. Next, number every other vertex by the distance from one of these three to it, and for each $i$, label uniformly at random the vertices at distance $i$ by $a\_i$, $b\_i$, and $c\_i$ and store in $X$ the edges between $a\_{i-1},b\_{i-1},c\_{i-1}$ and $a\_i, b\_i,c\_i$; i.e. \[X={\color{red} \bigcirc_{i=0}^{n-1}(} a\_{\color{red} i}>{\color{red} e(a\_i)}\_{\color{red} (i+1)};b\_{\color{red} i}>{\color{red} e(b\_i)}\_{\color{red} (i+1)};c\_{\color{red} i}>{\color{red} e(c\_i)}\_{\color{red} (i+1)} {\color{red} )} a\_0?b\_0?c\_0\]
where $e(v)$ represents the vertex that $v$'s edge points to, all of the instances of $i$ or $i+1$ should have the appropriate value substituted in and the symbols in black should be used exactly as stated. See Figure \ref{fig:3cycles} for an example. Finally, let $Y$ report whether $a\_0, b\_0,c\_0$ are in the same cycle or not. Now, consider training a T-regular neural network with a scratchpad of polynomial length on $(X,Y)$ generated in this manner. For any given $(X,Y)$, we will regard the net's loss on $(X,Y)$ as the expectation over all possible scratchpads that it might generate on $X$ of the loss of its eventual output. If we train it on $(X,Y)$ using population \footnote{This would also be true for batch GD with batches of size $n^c$ with $c$ chosen as a function of the other hyperparameters.} gradient descent with polynomial hyperparameters \footnote{I.e., either polynomial learning rate, polynomial clipping \cite{abbe2020poly, abbe2021power}, and weights stored using a logarithmic number of bits of precision and random rounding: for $a<b<c$ if $b$ needs to be rounded to $a$ or $c$ then it rounds to $c$ with probability $(b-a)/(c-a)$, or with polynomial learning rate, polynomial clipping and polynomial noise added to the gradients.}  and a differentiable loss function then the network fails to weakly learn to compute $Y$.
\end{theorem}

The proof of this theorem is not meaningfully different from the proof of the previous version, but for completeness we include it below.
\begin{proof}
First, observe that in the setup in the theorem, an input consists of a series of blocks where the $i$th block specifies where the edges from $a_{i-1}$, $b_{i-1}$, and $c_{i-1}$ go. These edges always go to $a_i$, $b_i$, and $c_i$ in some order (Regarding $a_n$, $b_n$, and $c_n$ as alternate names for $a_0$, $b_0$, and $c_0$). So, each block can be viewed as a permutation in $S_3$. Call these permutations $\sigma_1,...,\sigma_n$. The first $n-1$ of these permutations are mutually independent, but the requirement that there be $1$ or $3$ cycles rather than $2$ forces the last one to take on a value for which their product is even. If there are $3$ cycles their product is the identity and if there is one cycle their product is one of the other even permutations with which it is determined by whether the vertex $n$ edges from $a_0$ is $b_0$ or $c_0$. So, the probability distribution of $\sigma$ is the uniform distribution on the subset of $S_3^n$ with an even number of odd permutations and the label specifies whether or not $\prod_{i=1}^n\sigma_i$ is the identity.

The positional encodings are iid and there is no causal masking of the original input, so the probability distribution of the Transformer is symmetric under permutations of the inputs. In particular, for an arbitrary permutation $\pi\in S_n$ we can move all of the letters in block $i$ to the corresponding position in block $\pi(i)$ for each $i$ to essentially apply this permutation to the $\sigma$. Now, let $Y$ be the indicator function for $\prod_{i=1}^n\sigma_i$ being the identity. Then, for every $z\in\{0,1\}^{\lfloor (n-1)/2\rfloor}$ let $\pi_z=\prod_{i\le \lfloor (n-1)/2\rfloor:z_i=1} (2i-1,2i)$ and $Y_z$ be the indicator function for $\prod_{i=1}^n\sigma_{\pi_z(i)}$ being the identity. By the symmetries of the Transformer, it will have as hard a time learning to compute $Y$ as it would have to learn to compute $Y_z$ for any $z$. Next, let $Y_\emptyset$ be $1$ with probability $1/3$ and $0$ with probability $2/3$, independently of $X$. We claim that the probability distribution of the Transformer's weights when it is trained on $Y_z$ will be essentially the same as its weights when it is trained on $Y_\emptyset$. In order to formalize that, let $T_{w}(x)$ be a random output the Transformer with scratchpad gives when its edge weights and positional encoding take on the values given by $w$ on an input of $x$. Next, let $P_{W,\emptyset,t}$ be the probability distribution of the weights and positional embeddings after $t$ time steps of training on $Y_\emptyset$ and $P_{W,z,t}$ be the probability distribution of the weights and positional embeddings after $t$ time steps of training on $Y_z$. For each $t$, $P_{W,\emptyset,t}$ is symmetric under permutations of the positional embeddings.

By the assumption on hyperparameters, we are using gradient descent with a clipped gradient with some $B$ polynomial in $n$ such that if there exist $(X,Y)$ for which any entry of the gradient of the loss with respect to the edge weights has absolute value higher than $B$ it is reduced to $\pm B$ when calculating the overall gradient. Let $()_B$ denote this clipping operator. For any given set of edge weights and positional embeddings $w$ and a random $z$, the expected square of the difference between the $i$th elements of the clipped gradient of the loss when the net is trained on $Y_z$ and the clipped gradient of the loss when the net is trained on $Y_\emptyset$ is
\begin{align*}
&\mathbb{E}_z\left[\left(\mathbb{E}_{X,T_w(X)}\left[\left(\frac{d L(Y_z,T_w(X))}{d w_i}\right)_B\right]-\mathbb{E}_{X,T_w(X),Y_\emptyset}\left[\left(\frac{d L(Y_\emptyset,T_w(X))}{d w_i}\right)_B\right]\right)^2\right]\\
&=\mathbb{E}_z\left[\mathbb{E}^2_{X,T_w(X)}\left[(Y_z-1/3)\cdot\left(\left(\frac{d L(1,T_w(X))}{d w_i}\right)_B-\left(\frac{d L(0,T_w(X))}{d w_i}\right)_B\right) \right]\right]\\
&\le (1/3)^2(2B)^2\cdot 3\cdot 2^{-n/6}=(4B^2/3)\cdot 2^{-n/6} 
\end{align*}
by the previous corollary. There are a polynomial number of weights and a small difference in gradients results in a total variation distance between the probability distributions of the weights one step later than is at most polynomially larger, so for any $t$, we have that
\[\mathbb{E}_z [TV(P_{W,z,t+1},P_{W,\emptyset,t+1})]\le \mathbb{E}_z [TV(P_{W,z,t},P_{W,\emptyset,t})]+poly(n)\cdot 2^{-n/12} \]
That in turn means that
\[TV(P_{W,0,t},P_{W,\emptyset,t})=\mathbb{E}_z [TV(P_{W,z,t},P_{W,\emptyset,t})]=poly(n)\cdot 2^{-n/12}\]
for all $t$ polynomial in $n$. Now, let $y_0$ be chosen to minimize the value of $(1/3)L(1,y_0)+(2/3)L(0,y_0)$. Then, let $L'(p,q)=\min(L(p,q),\max(L(1,y_0),L(0,y_0)))$ for all $p$ and $q$ be a bounded version of the loss function. The previous argument shows that  the expected bounded loss of the Transformer on $(X,Y)$ after $t$ training steps is within $poly(n)\cdot 2^{-n/12}$ of the expected bounded loss of a Transformer trained on $(X,Y_\emptyset)$ for $t$ steps and then tested on $(X,Y)$ which is
\begin{align*}
&\mathbb{E}_{w\sim P_{W,\emptyset,t}}[\mathbb{E}_{X,T_w(X)}[L'(Y,T_w(X)) ]]\\
&=\mathbb{E}_{w\sim P_{W,\emptyset,t}}[\mathbb{E}_z[\mathbb{E}_{X,T_w(X)}[L'(Y,T_w(X)) ]]]\\
&=\mathbb{E}_{w\sim P_{W,\emptyset,t}}[\mathbb{E}_{X,T_w(X),Y_\emptyset}[L'(Y_\emptyset,T_w(X)) ]]\\
&\phantom{xxxxxx}+\mathbb{E}_{w\sim P_{W,\emptyset,t}}[\mathbb{E}_{z,Y_\emptyset}[\mathbb{E}_{X,T_w(X)}[L'(Y,T_w(X))-L'(Y_\emptyset,T_w(X)) ]]]\\
&= \mathbb{E}_{w\sim P_{W,\emptyset,t}}[\mathbb{E}_{X,T_w(X),Y_\emptyset}[L'(Y_\emptyset,T_w(X)) ]]\\
&\phantom{xxxxxx}+\mathbb{E}_{w\sim P_{W,\emptyset,t}}[\mathbb{E}_{z}[\mathbb{E}_X[(Y-1/3)\cdot(L'(1,T_w(X))-L'(0,T_w(X))) ]]]\\
&\ge \mathbb{E}_{w\sim P_{W,\emptyset,t}}[\mathbb{E}_{X,T_w(X),Y_\emptyset}[L'(Y_\emptyset,y_0) ]]\\
&\phantom{xxxxxx}-\mathbb{E}_{w\sim P_{W,\emptyset,t}}\left[\sqrt{\mathbb{E}_z[\mathbb{E}^2_{X,T_w(X)}[(Y-1/3)\cdot(L'(1,T_w(X))-L'(0,T_w(X))) ]]}\right]\\
&=\mathbb{E}_{X,T_w(X),Y_\emptyset}[L(Y_\emptyset,y_0) ]-O(2^{-n/12})
\end{align*}
as desired.
\end{proof}

\section{Comment on Lemma \ref{glob-cycle}}\label{app:comment-on-lemma1}
For $S$ such that $|S|<n$, $X[S]$ is independent of $Y$, since the distribution of such subsets of edges is the same for both classes. 

Let $S$ be such that $|S|=n$.
Let $Z_S$ be the ternary random variable that records whether there is a cycle or an open path on $S$.   Then, 
\begin{align}
 &I(X[S];Y) = I(Z_S;Y) 
\end{align}
since $I(X[S];Y) - I(Z_S;Y) =H(Y|X[S])-H(Y|Z_S)$ and $H(Y|X[S])=H(Y|Z_S)=H(Y|X[S],Z_S)$ since $Z_S$ contains all the information about $X[S]$ that is dependent on $Y$. Hence, 
\begin{align}
 &I(X[S];Y) = I(Z_S;Y)= H(Y)-H(Y|Z_S)\\& = 1 -H(Y|Z_S =1 )P(Z_S=1) -H(Y|Z_S =2 )P(Z_S=2) -H(Y|Z_S=0) P(Z_S=0) \\ &\sim  1-0-
(1-P(Z_S\ne 0)) = P(Z_S\ne 0). 
\end{align}
Now, 
\begin{align}
 P(Z_S\ne 0) &= P(\exists \text{ a cycle on $S$} ) + P(\exists \text{ an open path on $S$} ).
\end{align}
There is a cycle on $S$ if the graph is sampled from the two-cycle distribution and there are only two possible choices of cycles versus $\binom{2n}{n}$ possible selections of  edges, so 
\begin{align}
P(\exists \text{ a cycle on $S$} ) = 2/\binom{2n}{n}.
\end{align}
There is an open path on $S$ if the graph is sampled from the one-cycle distribution and there are $2n$ possible selections of such paths for $\binom{2n}{n}$ possible selections of the edges, so 
\begin{align}
P(\exists \text{ an open path on $S$} ) = 2n/\binom{2n}{n}.
\end{align}
Thus 
\begin{align}
 P(Z_S\ne 0) &= (2+2n)/\binom{2n}{n} \sim \frac{\sqrt{\pi}}{2} (2n)^{3/2} 2^{-2n}.
\end{align}
Therefore, even for sets of size $n$, the mutual information is exponentially low, implying that $\mathrm{glob}(D)$ is greater than $n+1$.

\section{Discussion on circuit complexity connections}\label{app:circuit-complexity-connections}
One approach we can use to analyze what we can learn with different methods is to consider the complexity class of the problems that can be solved by algorithms of a given type. Constant depth neural nets with well-behaved activation functions and weights of size at most polynomial in the input length are limited to computing functions in (possibly nonuniform) $TC^0$, the class of functions computable by polynomial-sized constant depth circuits built from AND, OR, NOT, and threshold gates. Likewise, constant depth Transformers with polynomial-sized weights, polynomial-sized alphabets, and attention matrices whose entries are rational with polynomial-sized numerators and denominators are limited to computing functions in $TC^0$.

The next circuit complexity class above $TC^0$ is $NC^1$, the class of functions computable by polynomial-sized circuits of logarithmic depth that are built from AND gates on pairs of values, OR gates on pairs of values, and NOT gates. It has not been proven that $NC^1\ne TC^0$, but it is suspected to be the case. Among other things, the problem of determining whether a product of permutations in $S_5$ is the identity permutation or some other specified even permutation is $NC^1$-complete, so $TC^0$ circuits cannot compute it unless $TC^0=NC^1$. Furthermore, given any permutations $\sigma_1,\sigma_2,...,\sigma_n\in S_5$ and random $r_0,...,r_n\in S_5$, it is the case that $\sigma_1\cdot \sigma_2\cdot...\cdot\sigma_n=r_0^{-1}(r_0\sigma_1 r_1^{-1})\cdot...\cdot(r_{n-1}\sigma_n r_n^{-1})r_n$ and $(r_0\sigma_1 r_1^{-1}),...,(r_{n-1}\sigma_n r_n^{-1})$ is a random string of permutations. So, computing the product of a string of permutations is essentially as hard in the average case as in the worst case.

Now, one of these products of permutations can be converted to a graph as follows. First, the graph has a set of $5$ vertices representing the $5$ possible inputs to the product, and another set of $5$ vertices representing the $5$ possible outputs of each permutation in the product. Each vertex has an edge to the vertex representing the value the next permutation in the product maps its value to, and there are edges from each of the final $5$ vertices to the corresponding one of the first $5$ vertices. If the product of permutations is the identity, then this graph consists of $5$ cycles of length $n+1$, while otherwise it has a smaller number of cycles. So, determining whether or not the product of permutations is the identity can be reduced to determining which pairs of the first $5$ vertices are in the same component. Thus, if $TC^0\ne NC^1$ then constant-depth neural nets and Transformers cannot determine whether or not two vertices in an arbitrary graph are in the same component with nontrivial accuracy.

On the other hand, with appropriate setup, deep neural nets, recurrent neural nets, and Transformers with scratchpads are Turing complete. Furthermore, they can simulate a Turing machine using resources polynomial in the number of steps the Turing machine runs for and the input length. So, with appropriate parameters these can efficiently solve any problem that it is possible to solve efficiently. A little more precisely, given a neural net where the input bits are $0$ or $1$, it is fairly easy to set a neuron to compute an AND, OR, or NOT of one or more previous values, so any circuit can be converted into a neural net of at most equal size. Any efficient computation can be performed by a polynomial-sized circuit, so it can also be performed by a polynomial-sized deep neural net. Also, given a Turing machine in a state where all entries in its tape that are more than $n$ steps away from the head or heads are in their initial state, there is a circuit of depth $O(1)$ and size $O(n)$ that computes the next state of the Turing machine. That means that running a Turing machine for $T$ steps on an input of length $n$ can be simulated by a recurrent neural net of size $O(T+n)$ and $T$ recurrences. Conversely, given a neural net with a reasonable activation function and subexponential edge weights, one can estimate the output of each neuron to within an exponentially small error in time polynomial in the size of the net.

The topic of the capabilities of a Transformer with a scratchpad is a bit more complicated. The work of \cite{expressivity} analyses the capabilities of a Transformer with a constant-sized alphabet, constant depth, intermediate variables expressible in logarithmic numbers of bits, causal masking, and a form of hard attention where self-attention operations always average over all previous entries that maximize the attention score. It shows that given an input of length $n$ and scratchpad of length $T$ such a Transformer can perform any computation doable in time $T$, and conversely that any computation such a Transformer can perform is doable in $O((T^2+n^2)polylog(T+n))$ time and $O(T+\log(n))$ space. They note that this means that a Transformer with a logarithmic length scratchpad is limited to performing computations in logspace, while a Transformer with a scratchpad of linear length can simulate a finite state machine, and a Transformer with a suitably long polynomial length scratchpad can perform any computation in $P$.

{\bf $TC^0$ versus logspace limitation.} We now tighten the logspace result from \cite{expressivity}.

\begin{lemma}\label{logspace}
    A constant-depth Transformer with intermediate values recorded to inverse-polynomial accuracy, a logarithmic length scratchpad, and a constant alphabet size is still limited to computing functions in $TC^0$. 
\end{lemma}
\begin{proof}
First, recall that one can compute the probability distribution of the Transformer's output on any given input in $TC^0$ by a result of \cite{parallelism}. So, for any given partial scratchpad one can determine the probability distribution of the Transformer's output for the scratchpad's next entry with a $TC^0$ function. That means that one can find the probability that the Transformer would generate any given scratchpad in $TC^0$ by checking the probability that it outputs each entry when run on the previous scratchpad entries and the original input and then multiplying them to inverse-polynomial accuracy. There are only a polynomial number of possible strings the Transformer could write in its scratchpad, so a $TC^0$ circuit can check them all in parallel in order to determine how likely they each are and then add up the contributions to the probability of each output from each possible scratchpad in order to determine the probability distribution of the Transformer's output.

%Actually, such a transformer with scratchpad is still limited to computing functions in $TC^0$. To see this, first observe that since the alphabet size is constant there are only a polynomial number of possible strings the transformer could write in the scratchpad. So, one can have a $TC^0$ circuit such that for every possible scratchpad $s$ and every index $i$ the circuit checks in parallel if the transformer would write $s_i$ in the scratchpad if the previous values in the scratchpad were $s_1,...,s_{i-1}$. Then it can use those to determine what it would actually write in the scratchpad and what it would output.
\end{proof}

{\bf Connection between globality degree and $NC^0$.} We next show that, putting aside the histogram knowledge and using constant alphabet size, having low globality means having correlations with $NC^0$ circuits, a class of circuit weaker than $TC^0$ by constraining the number of fan-in to be constant (and not allowing for threshold gates). 
\begin{lemma}\label{NC0_lemma}
Let $P_X$ be a probability distribution over $\{0,1\}^n$ and $f\colon\{0,1\}^n\rightarrow\{0,1\}$ be a function. $f$ correlates non-trivially with a function computable in $NC^0$ if and only if $\mathrm{glob}(f)=O_n(1)$ when $X\sim P_X$.
\end{lemma}

\begin{proof}
The $NC^0$ functions are exactly the binary functions that only depend on a constant number of input bits. So, for any $NC^0$ function $g$, there exists a set of input bits $S$ such that $|S|=O(1)$ and $g(X)$ only depends on the restriction of $X$ to $S$. That means that if $f$ has a nontrivial correlation with $g$ then $f(X)$ has a nontrivial correlation with $X[S]$ and thus has globality at most $|S|=O(1)$.

Now, assume that $\mathrm{glob}(f)=O_n(1)$ and choose $S$ with $|S|=O(1)$ such that $I(X[S],f(X))=n^{-O(1)}$. Next, define the function $g'\colon\{0,1\}^n\rightarrow[0,1]$ such that $g'(x)=\mathbb{E}[f(X)|X[S]=x[S]]$ for all $x$ and $g\colon\{0,1\}^n\rightarrow\{0,1\}$ such that 
\[g(x)=\begin{cases}
1 &\text{if } g'(x) \ge \mathbb{E}[f(X)]\\
0 &\text{otherwise}
\end{cases}\]
for all $x\in\{0,1\}^n$. This function is in $NC^0$ because $g(x)$ depends only on $x[S]$. Furthermore, the correlation between $f$ and $g$ is
\begin{align*}
&covar(f,g)/\sqrt{var(f) var(g)}\\
&\ge covar(f,g)\\
&=\mathbb{E}_X\left[(g(X)-\mathbb{E}[g(X')])(f(X)-\mathbb{E}[f(X')])\right]\\
&=\mathbb{E}_X\left[(g(X)-\mathbb{E}[f(X')])(f(X)-\mathbb{E}[f(X')])\right]\\
&=\mathbb{E}_X\left[(g(X)-\mathbb{E}[f(X')])(g'(X)-\mathbb{E}[f(X')])\right]\\
&\ge \mathbb{E}_X\left[(g'(X)-\mathbb{E}[f(X')])^2\right]\\
&=n^{-O(1)}
\end{align*}
as desired.
\end{proof}

\section{Experiments with ChatGPT}\label{app:gpt}
\paragraph{Height comparison.}
For $n \geq 1$, we consider $3n+2$ people having different heights. We give the model $3n+1$ pairwise relations between the consecutive people (in order of height) in a random order. Using this information, one can understand the order of the heights for all people by combining the given information. We ask the model about the relation between person $n+1$ and $2n+2$. An example for $n=1$ is 
\begin{center}
    ``Omar is taller than Sara. Vlad is taller than David. Farah is taller than Omar. Sara is taller than Vlad. Is Omar taller than Vlad?"
\end{center}
where the answer is true. Note that to answer this question correctly one has to combine at least $n+1$ relations. Thus, the globality of the task is always larger than $n$. (The exact globality would depend on the tokenization.) We found out that ChatGPT (GPT3.5) fails at this task even for $n=1$ (simplest case). Note that when working with the GPT3.5 model we used the following prompt so that the model is able to use chain-of-thought reasoning: "You can reason if you want but make sure to include yes/no in your answer." Interestingly, GPT4 performs much better than GPT3.5. We also observed that it is often the case that when GPT4 answers correctly to the question, it orders people based on their height, very similar to what we do in the scratchpad of the graph task. Motivated by this, we tested one more setting where we prompted GPT4 with "Answer only with a yes or no." to avoid the chain-of-thought reasoning. In this case, as expected, the model couldn't solve the height comparison task for $n>1$. The results are shown in Figure \ref{fig:chatgpt}.

\begin{figure}[H]
    \centering
    \includegraphics[width=0.6\textwidth]{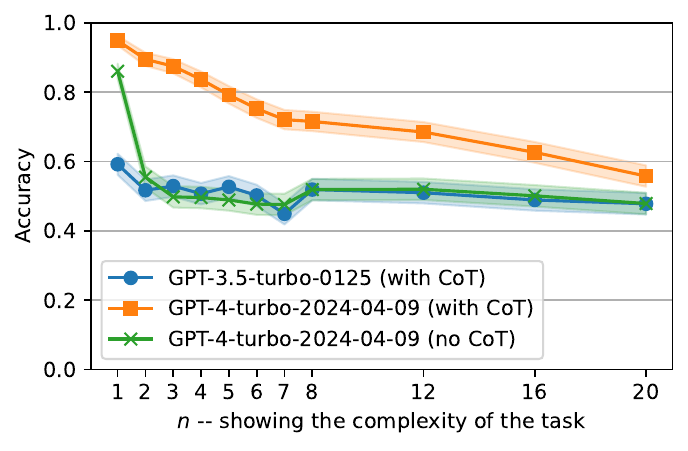}
    \caption{For complexity $n$ we have $3n+2$ people and there are $n$ people between the two names we query (see example above). We found out that ChatGPT(3.5) can hardly go beyond the random baseline on this task even for $n=1$ while GPT4 performs much better. However, if GPT4 does not use CoT reasoning, its performance would be near random for $n > 1$. Note that we used 1000 examples for each value of $n$.}
    \label{fig:chatgpt}
\end{figure}

\end{document}